\documentclass[10pt,journal,compsoc]{IEEEtran}
\ifCLASSOPTIONcompsoc
  \usepackage[nocompress]{cite}
\else
  \usepackage{cite}
\fi

\usepackage{epsfig}
\usepackage{graphicx,subfigure}
\usepackage{amsmath}
\usepackage{amssymb}
\usepackage{color}
\usepackage{booktabs}
\usepackage{multirow}
\usepackage[table]{xcolor}
\usepackage{url}
\usepackage[ruled, linesnumbered]{algorithm2e}
\usepackage{xcolor}
\usepackage{enumitem}
\usepackage{amsthm}
\usepackage{empheq}
\usepackage{latexsym}
\usepackage{tabularx}
\usepackage{array}
\usepackage{diagbox}
\usepackage{listings}

\definecolor{commentcolor}{rgb}{0,0,0}
\newtheorem{theorem}{Theorem}
\newtheorem{lemma}{Lemma}

\newtheorem{proposition}{Proposition}

\newenvironment{mytheorem}[1]{%
	\renewcommand\thetheorem{#1}
	\theorem
}{\endtheorem}

\def\a{\mathbf{a}}
\def\b{\mathbf{b}}
\def\c{\mathbf{c}}
\def\A{\mathbf{A}}

\def\d{\mathbf{d}}

\def\E{\mathbf{E}}
\def\K{\mathbf{K}}

\def\R{\mathbf{R}}
\def\M{\mathbf{M}}

\def\n{\mathbf{n}}
\def\p{\mathbf{p}}
\def\Q{\mathbf{Q}}

\def\x{\mathbf{x}}
\def\y{\mathbf{y}}

\def\Q{\mathbf{Q}}

\DeclareMathOperator*{\trace}{trace}

\DeclareMathOperator*{\quot}{quot}
\DeclareMathOperator*{\rank}{rank}

\ifCLASSINFOpdf
\else
\fi
\begin{document}
%
\title{On Relative Pose Recovery for Multi-Camera Systems}
%
%
%
%

\author{Ji Zhao \ and \ Banglei Guan
\IEEEcompsocitemizethanks{\IEEEcompsocthanksitem J. Zhao is in Beijing, China.
	\protect\\
E-mail: zhaoji84@gmail.com
\IEEEcompsocthanksitem B. Guan is with College of Aerospace Science and Engineering, National University of Defense Technology, Changsha 410073, China.\protect\\
E-mail: guanbanglei12@nudt.edu.cn
}
\thanks{(Corresponding author: Ji Zhao)}
}

%
%

\markboth{}
{Shell \MakeLowercase{\textit{et al.}}: Bare Demo of IEEEtran.cls for Computer Society Journals}
%



\IEEEtitleabstractindextext{%
\begin{abstract}
The point correspondence (PC) and affine correspondence (AC) are widely used for relative pose estimation. An AC consists of a PC across two views and an affine transformation between the small patches around this PC. Previous work demonstrates that one AC generally provides three independent constraints for relative pose estimation. For multi-camera systems, there is still not any AC-based minimal solver for general relative pose estimation. To deal with this problem, we propose a complete solution to relative pose estimation from two ACs for multi-camera systems, consisting of a series of minimal solvers. The solver generation in our solution is based on Cayley or quaternion parameterization for rotation and hidden variable technique to eliminate translation. This solver generation method is also naturally applied to relative pose estimation from PCs, resulting in a new six-point method for multi-camera systems. A few extensions are made, including relative pose estimation with known rotation angle and/or with unknown focal lengths. Extensive experiments demonstrate that the proposed AC-based solvers and PC-based solvers are effective and efficient on synthetic and real-world datasets.
\end{abstract}

}

\maketitle

\IEEEdisplaynontitleabstractindextext

%
\IEEEpeerreviewmaketitle

\IEEEraisesectionheading{\section{Introduction}}

\IEEEPARstart{E}{stimating} the relative pose of images given feature correspondences is one foundational task in geometric vision. It is crucial for many popular tasks, such as structure-from-motion, simultaneous localization and mapping, autonomous driving, augmented reality, etc. 
Despite its long history, relative pose estimation is still an active research area. A lot of methods have been developed to improve its accuracy, efficiency, numerical stability, and robustness.

Camera models play an important role in relative pose estimation. While a single camera can be modeled by a pinhole or perspective camera model~\cite{hartley2003multiple}, more complicated cameras such as multi-camera systems should be modeled
by the generalized camera model~\cite{pless2003using}. A generalized camera is formed by abstracting landmark observations into spatial rays that are no longer required to originate from a common point (i.e., the focal point). This paper deals with both a single camera and a multi-camera rig of rigidly attached cameras. Multi-camera systems are popular in autonomous driving~\cite{geiger2013vision} and many robotic applications.
It is well-known that the standard epipolar geometry exists an unobservability in the scale of the translation~\cite{hartley2003multiple}. In contrast, the scale of translation estimated for a multi-camera system is generally unique. The down-side of this scale observability is that the minimal solution of the relative pose requires at least $6$ instead of only $5$ point correspondences (PCs) across the two views.

Since feature correspondences inevitably contain outliers, a robust estimator should be applied to recover the correct pose and remove outliers. In computer vision and robotics communities, the most popular framework is random sample consensus (RANSAC)~\cite{fischler1981random} and its variants. The core component in RANSAC is a minimal solver. 
Minimal solvers for relative pose estimation usually rely on PCs. According to the well-known epipolar geometry, each PC provides one constraint for a relative pose.
For single cameras, the representative methods for relative pose estimation are called five-point methods~\cite{nister2004efficient,stewenius2006recent,li2006five,kneip2012finding,Kukelova12polynomial,fathian2018quest}.
For multi-camera systems, the development of minimal solvers for relative pose estimation ranges back to the six-point method~\cite{stewenius2005solutions}. Later, many other works have been subsequently proposed, such as the 17-point linear solvers~\cite{li2008linear}, an efficient solver based on iterative optimization~\cite{kneip2014efficient}, and a solver based on global optimization~\cite{zhao2020certifiably}.

The efficiency of RANSAC depends on the runtime of the minimal solver and iteration number. 
Using fewer feature correspondences in minimal solvers will cause fewer iterations. 
Thus much effort has been made to reduce the required number of feature correspondences in minimal solvers while keeping their efficiency. 
In recent years, there are a series of work exploiting affine correspondences (ACs) to estimate relative pose~\cite{bentolila2014conic,raposo2016theory,barath2017minimal,barath2018efficient,eichhardt2018affine,barath2020making}. Since one AC provides more independent constraints than one PC, fewer ACs than PCs are needed to solve pose estimation problems. Several minimal solvers using ACs have been developed for different tasks. For single cameras, there are pose estimation methods for general 5~degrees of freedom (DOF) motion~\cite{barath2017minimal}, 3DOF motion with known rotation axis~\cite{guan2019minimal}, and 2DOF motion with planar motion~\cite{guan2019minimal,hajder2020relative}. For multi-camera systems, there are pose estimation methods for motion with known rotation axis and planar motion~\cite{guan2020relative}.

\begin{figure}[tpb]
	\begin{center}
		\subfigure[inter-camera ACs]
		{
			\includegraphics[width=0.47\linewidth]{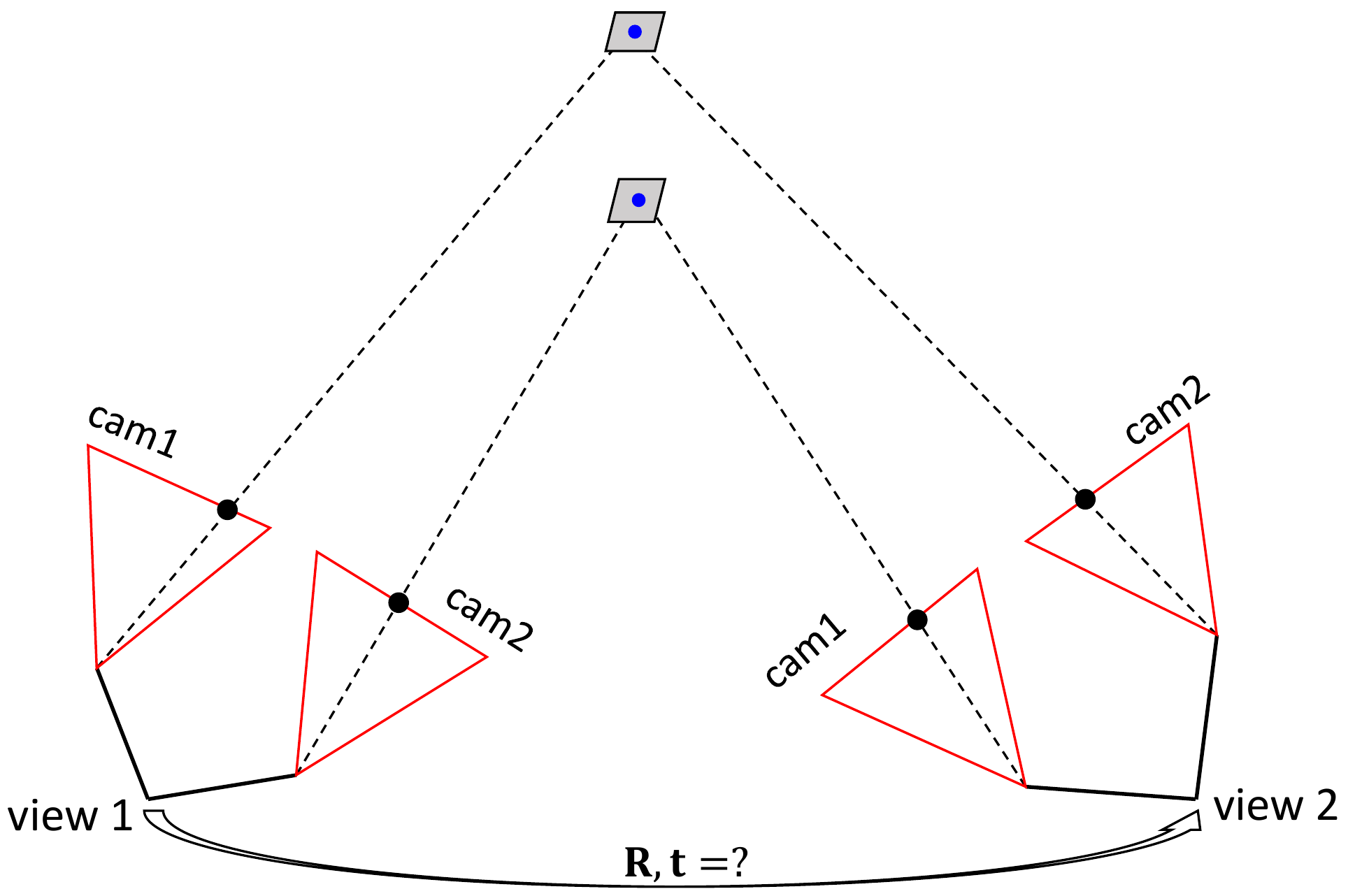}
		}
		\subfigure[intra-camera ACs]
		{
			\includegraphics[width=0.47\linewidth]{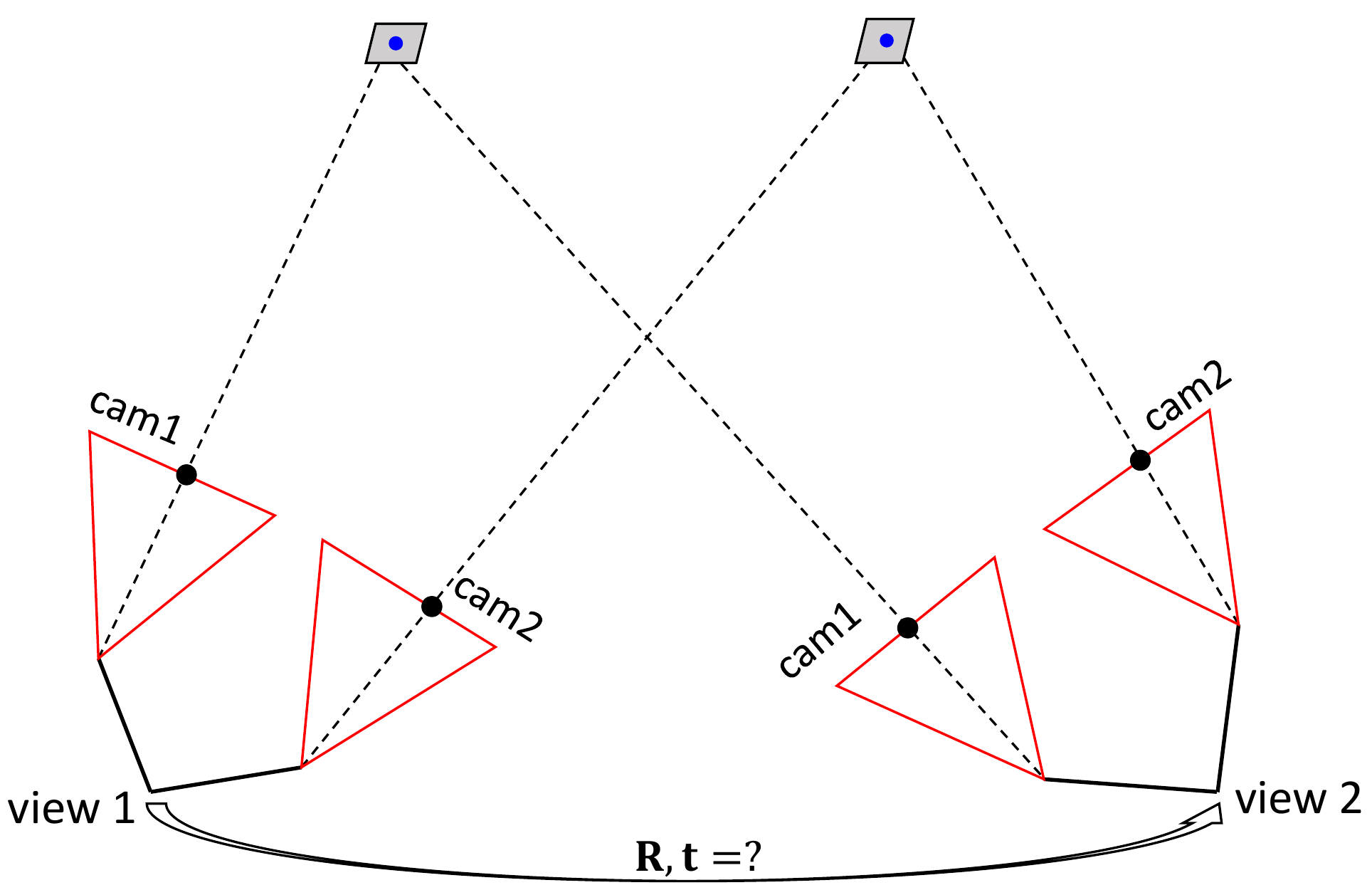}
		}
	\end{center}
	\vspace{-0.1in}
	\caption{Relative pose estimation from ACs for a multi-camera system. Two oriented points (or infinitesimal patches equivalently) are captured by two views of a two-camera rig. We aim to recover the 6DOF relative pose using two ACs.}
	\label{fig:teaser_ac_pose}
\end{figure}

This paper mainly deals with the full DOF relative pose estimation for multi-camera systems, see Fig.~\ref{fig:teaser_ac_pose}. Due to the diversity of camera layout for multi-camera systems, the minimal solvers are complicated. 
Fig.~\ref{fig:teaser_ac_pose} demonstrates the relative pose estimation from inter-camera ACs and intra-camera ACs. The corresponding solvers for these two configurations are different.
In this paper, the proposed AC-based minimal solvers form a complete solution to relative pose estimation for multi-camera systems. 
We also demonstrated that our method is versatile to be applied to PCs and many extensions straightforwardly.
The contributions of this paper are three-fold.
\begin{itemize}
	\item We mainly deal with full DOF relative pose estimation from ACs for multi-camera systems. A series of minimal solvers are proposed, which cover all camera layouts and AC types of this problem. The proposed minimal solvers form a complete solution to this problem.
	To the best of our knowledge, this is the first solution that can provide minimal solvers to this problem.

	\item Using the same equation system construction and solver generation, two minimal solvers are proposed for full DOF relative pose estimation from PCs. 
	They are new six-point methods for popular configurations of two-camera rigs.

	\item Our solver generation is based on a unified and versatile framework, which exploits Cayley or quaternion parameterization for rotation and uses the hidden variable technique to eliminate translation. We proved that a factor should be factored out to avoid false roots and simplify the equation system. This framework has broad applications in relative pose estimation. In the supplementary material, some minimal solvers are generated with a known rotation angle and/or unknown focal lengths. 
\end{itemize}

The paper is organized as follows. Section~\ref{sec:related_work} introduces the related work. 
In Section~\ref{sec:minimal_config},  minimal configurations and degenerate cases for relative pose estimation from ACs are figured out. 
For both single cameras and multi-camera systems, Section~\ref{sec:complete_solution} proposes a complete solution and a series of AC-based minimal solvers for relative pose estimation.
A few PC-based minimal solvers for multi-camera systems are provided in Section~\ref{sec:pt_pose}.  
Section~\ref{sec:experiment} presents the performance of our method in comparison to other methods, followed by a concluding discussion in Section~\ref{sec:conclusion}.

\section{Related Work}
\label{sec:related_work}
Relative pose estimation is an essential topic in computer vision, and there are many methods in this area. We introduce a taxonomy of these methods along four independent axes.

First, relative pose estimation methods can be applied for single cameras~\cite{nister2004efficient,stewenius2006recent,li2006five,kneip2012finding,Kukelova12polynomial,fathian2018quest}, multi-camera systems or generalized cameras~\cite{stewenius2005solutions,li2008linear,kneip2014efficient,zhao2020certifiably}. 
Moreover, the cameras include fully calibrated cameras and partially calibrated cameras with unknown focal length and/or radial distortion.

Second, the exploited geometric primitives can be points~\cite{nister2004efficient,stewenius2005solutions}, oriented points~\cite{bentolila2014conic,raposo2016theory}, hybrid of point and line structures~\cite{duff2019plmp,fabbri2020trplp,zhao2020minimal}, etc.
Recently, oriented points (or an infinitesimal patch equivalently) have drawn much attention. 
The images of an oriented point captured by different views form an AC between affine-covariant features. 
One AC provides a $2\times 2$ affine transformation matrix between the local patches around the corresponding points in two images.  Two-view geometry of AC has been investigated~\cite{bentolila2014conic,raposo2016theory,barath2018efficient,eichhardt2018affine}, and several methods of relative pose from ACs has been proposed~\cite{bentolila2014conic,raposo2016theory,barath2017minimal,barath2018efficient,eichhardt2018affine,barath2020making}. 
In addition, there are relative pose estimation methods~\cite{barath2018five,barath2018recovering,barath2019homography} exploit orientation- and scale-covariant features obtained by e.g. the SIFT detector~\cite{lowe2004distinctive}.

Third, a relative pose estimation method can be classified as a minimal solver~\cite{nister2004efficient,stewenius2005solutions}, a non-minimal solver~\cite{kneip2013direct,kneip2014efficient,zhao2020certifiably,zhao2020efficient} or a linear solver~\cite{hartley1997defence,li2008linear}. 
Minimal solvers use the minimum required number of geometric primitives to estimate relative pose.
Non-minimal solvers use all correct correspondences to estimate an accurate relative pose. The linear solvers also use larger-than-minimal correspondences, but their primary purpose is finding a simple and efficient solution. Linear solvers usually ignore certain implicit constraints for unknowns to gain efficiency. In contrast, minimal solvers and non-minimal solvers consider all the implicit constraints for unknowns. This paper focuses on minimal solvers.

Fourth, existing methods might estimate the pose with or without a prior. One popular scene prior is the planar assumption. If 3D points lie on a plane, the relative pose can be recovered by estimating homography using 4~PCs~\cite{hartley2003multiple} or 1AC+1PC~\cite{barath2017theory}. 
Usually, the pose priors include known rotation axis prior, known rotation angle prior, planar motion prior, or Ackermann motion prior.  
Common ways to obtain the directional direction are given by vanishing point estimation or sensor fusion with an inertial measurement unit (IMU) to measure the direction of gravity. The known rotation axis prior can reduce the DOF of rotation by two~\cite{naroditsky2012two,lee2014relative,sweeney2014solving,saurer2016homography,ding2020homography,ornhag2020minimal}. 
The known rotation angle prior comes from a camera-IMU module with unknown extrinsic parameters, which reduces the DOF of rotation by one~\cite{li2013point,li2020relative,martyushev2020efficient}. 
The planar motion prior can reduce the DOF of rotation and translation by two and one, respectively~\cite{choi2018fast}.
The Ackermann motion prior reduces the DOF to one for a single camera~\cite{scaramuzza20111} and two for a multi-camera system~\cite{lee2013motion}.
In addition, there is a small-rotation prior. In this case, the rotation can be approximated by its first-order approximation~\cite{ventura2015efficient}.

The combination of the four aspects mentioned above produces many subcategories. 
Providing a comprehensive survey is beyond the scope of this paper. In the following text, we focus on four subcategories from the combination of two camera types (single camera or multi-camera system) and two correspondence types (PC or AC).

{\bf Single camera and PCs}: Usually, five PCs are enough to estimate the relative pose of a single camera. In addition to the well-known five-point method~\cite{nister2004efficient}, there are many other five-point methods, such as~\cite{stewenius2006recent,li2006five,kneip2012finding,Kukelova12polynomial,fathian2018quest}.
There are a lot of solvers with motion priors, such as known rotation axis~\cite{naroditsky2012two,sweeney2014solving,saurer2016homography}, known axis angle~\cite{li2013point,li2020relative,martyushev2020efficient}, planar motion~\cite{choi2018fast}, and Ackermann motion~\cite{scaramuzza20111}. 
If the camera is uncalibrated, there are some solvers that simultaneously estimate pose and intrinsic parameters, such as focal length~\cite{stewenius2005minimal,hartley2012an,Kukelova12polynomial,kukelova2017clever} and/or radial distortion~\cite{kukelova2011minimal,kuang2014minimal,kukelova2017clever}.

{\bf Single camera and ACs}: 
One AC usually provides $3$ independent constraints on relative pose~\cite{bentolila2014conic}. As a result, two ACs are enough to determine the relative pose for a single camera~\cite{raposo2016theory,eichhardt2018affine,barath2018efficient} or a single camera with unknown focal length~\cite{barath2017minimal}. 
In~\cite{barath2020making}, guidelines are proposed for effective usage of ACs in the course of a full model estimation pipeline. 
When the relative pose has a motion prior, there are simple and efficient solvers. For example, there are many customized solvers when the rotation axis is known~\cite{guan2019minimal}, the motion is under planar motion~\cite{guan2019minimal,hajder2020relative}, or the depth of feature points is known~\cite{eichhardt2020relative}.

{\bf Multi-camera system and PCs}: The first minimal solver for multi-camera systems was proposed in~\cite{stewenius2005solutions}, which uses $6$ PCs to estimate the relative pose. A linear solver that takes $17$ PCs was proposed in~\cite{li2008linear}. Some solvers are developed in the context of structure-from-motion with special configurations~\cite{zheng2015structure,kasten2019resultant}. 
There are a few non-minimal solvers, which use local optimization~\cite{kneip2014efficient} or global optimization method~\cite{zhao2020certifiably} to find the optimal relative poses. There are also solvers with motion priors, such as known rotation axis~\cite{lee2014relative,sweeney2014solving,liu2017robust}, a known rotation angle~\cite{martyushev2020efficient}, and Ackermann motion~\cite{lee2013motion}. A first-order approximation to relative pose was used in~\cite{ventura2015efficient} to obtain an efficient solver.

{\bf Multi-camera system and ACs}: 
Using ACs to estimate the relative pose of multi-camera systems is a relatively new research area. For non-degenerate cases, 2~ACs are enough for 6DOF pose estimation. A linear solution using $6$~ACs was proposed in~\cite{alyousefi2020multi}, which generalizes the $17$-point solver for PCs~\cite{li2008linear}. 
Two minimal solvers were proposed for motion with known vertical direction or planar motion~\cite{guan2020relative}.

\section{Minimal Configurations for Pose Estimation from Affine Correspondences}
\label{sec:minimal_config}

In this section, the intrinsic and extrinsic parameters of multi-camera systems are assumed to be calibrated. Our purpose is to find all the minimal configurations for the relative pose estimation by using ACs.

\subsection{Two-View Geometry for A Single Camera}
Denote the $k$-th AC as $(\x_k, \x'_k, \A_k)$, where $\x_k$ and $\x'_k$ are the homogeneous coordinates in normalized image plane for the first and the second views, respectively. $\A_k$ is a $2\times 2$ local affine transformation, which relates the infinitesimal patches around $\x_k$ and $\x'_k$~\cite{raposo2016theory}.
Generally speaking, one AC provides three independent constraints: one is derived from the point correspondence $(\x_k, \x'_k)$, and two are derived from the affine transformation $\A_k$. 
Denote the relative rotation and translation from the first view to the second view as $\R$ and $\mathbf{t}$, respectively.

The constraint introduced by the point correspondence $(\x_k, \x'_k)$ is~\cite{hartley2003multiple}
\begin{align}
\x_k'^T \E \x_k = 0,
\label{eq:constraint_epipolar}
\end{align}
where
\begin{align}
\E = [\mathbf{t}]_\times \R.
\label{eq:essential}
\end{align}
This constraint is the well-known epipolar constraint in two-view geometry. $\E$ is known as the essential matrix.

The two constraints introduced by the affine transformation $\A$ can be written as~\cite{barath2018efficient}
\begin{align}
(\E^T \x_k')_{(1:2)} + \A_k^{T} (\E \x_k)_{(1:2)} = \mathbf{0},
\label{eq:constraint_affine}
\end{align}
where the subscript $(1:2)$ represents the first two entries of a vector. 

\subsection{Two-View Geometry for A Multi-Camera System}

\begin{figure}[tpb]
	\begin{center}
		\includegraphics[width=0.8\linewidth]{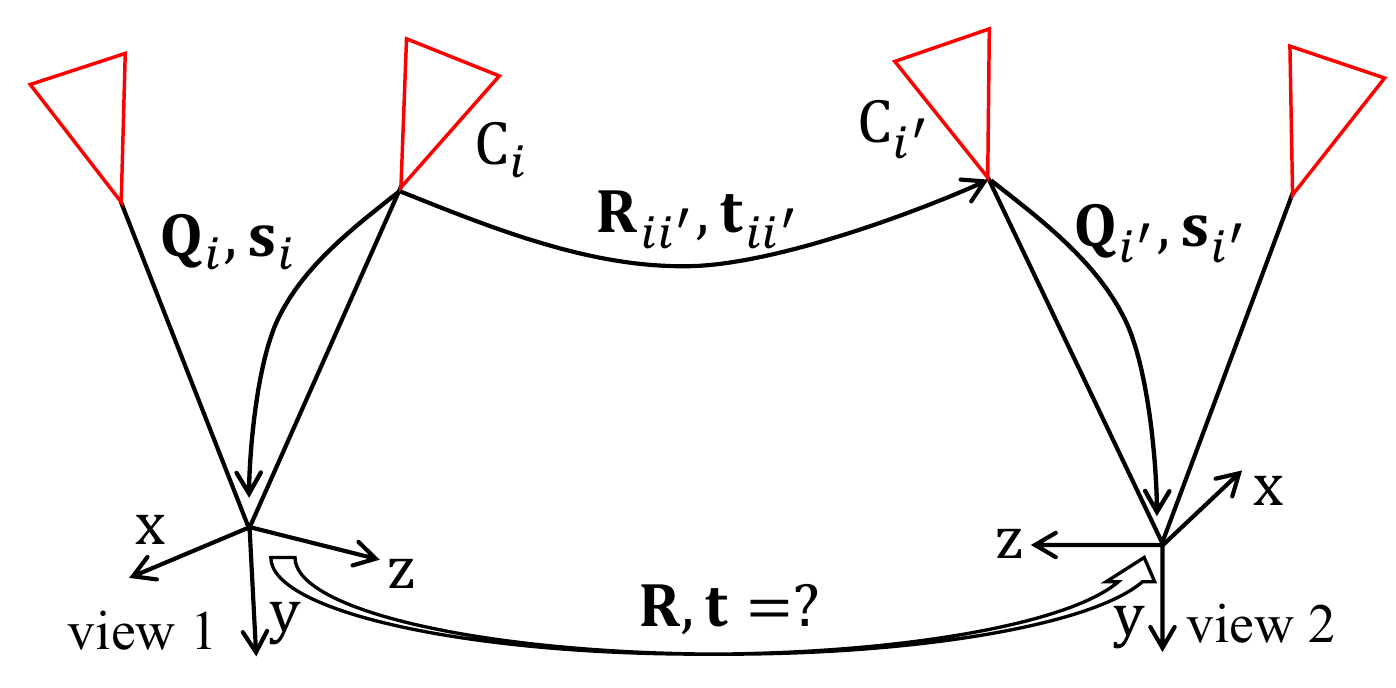}
	\end{center}
	\vspace{-0.1in}
	\caption{Illustration of the relative pose estimation for a multi-camera system.}
	\label{fig:multi_cam_system_ac}
\end{figure}

A multi-camera system is composed of several perspective cameras. Denote the extrinsic parameters of the $i$-th camera as $\{\Q_i, \mathbf{s}_i\}$. Here $\Q_i$ and $\mathbf{s}_i$ represent relative rotation and translation to the reference of the multi-camera system. 
Denote the relative rotation and translation from the first view to the second view of the multi-camera system as $\R$ and $\mathbf{t}$, respectively.

An affine correspondence in a multi-camera system relates two perspective cameras across two views, see Fig.~\ref{fig:multi_cam_system_ac}. Denote the $k$-th AC as $(\x_k, \x'_k, i_k, i'_k, \A_k)$. It means that an oriented point is captured by the $i_k$-th camera in the first view, and its homogeneous coordinate in the normalized image plane is $\x_k$. It is also captured by the $i'_k$-th camera in the second view, and its homogeneous coordinate is $\x'_k$. 
To simplify the notation, we omit the subscript $k$ of camera indices $i$ and $i'$ in the following text.

If we consider cameras $i$ and $i'$ of an AC as one perspective camera across two views, the constraints of Eqs.~\eqref{eq:constraint_epipolar}~and~\eqref{eq:constraint_affine} should still hold. However, the essential matrix $\E$, $\R$, and $\mathbf{t}$ in these two equations should be newly defined. In addition, the essential matrices for different ACs are usually different. 
Specifically, the constraints introduced by the $k$-th AC are 
\begin{subequations}
	\begin{empheq}[left=\empheqlbrace]{align}
	& \x_k'^T \E_k \x_k = 0 \label{equ:sub1} \\
	& (\E_k^T \x_k')_{(1:2)} + \A_k^{T} (\E_k \x_k)_{(1:2)} = \mathbf{0} \label{equ:sub2}
	\end{empheq}
\end{subequations}
where
\begin{align}
\E_k = [\mathbf{t}_{ii'}]_\times \R_{ii'}.
\label{eq:essential_gcam}
\end{align}
$\{\R_{ii'}, \mathbf{t}_{ii'}\}$ is the relative pose from camera $i$ in the first view to camera $i'$ in the second view. It is determined by a composition of three transformations
\begin{align}
\begin{bmatrix}
\R_{ii'} & {\mathbf{t}_{ii'}}\\
{{\mathbf{0}}}&{1}\\
\end{bmatrix} 
= &
\begin{bmatrix}
{\Q_{i'}}&{\mathbf{s}_{i'}}\\
{{\mathbf{0}}}&{1}\\
\end{bmatrix}^{-1}
\begin{bmatrix}
\R&{\mathbf{t}}\\
{{\mathbf{0}}}&{1}\\
\end{bmatrix}
\begin{bmatrix}
{\Q_{i}}&{\mathbf{s}_{i}}\\
{{\mathbf{0}}}&{1}\\
\end{bmatrix} \nonumber \\
= & \begin{bmatrix} {\Q_{i'}^T \R \Q_i} & \Q_{i'}^T (\R \mathbf{s}_{i} + \mathbf{t} - \mathbf{s}_{i'})\\
{{\mathbf{0}}}& \ {1}\\
\end{bmatrix}.
\label{eq:transformation}
\end{align}
By substituting $\R_{ii'}$ and $\mathbf{t}_{ii'}$ into Eq.~\eqref{eq:essential_gcam}, the essential matrix $\E_k$ becomes
\begin{align}
\E_k & = [\Q_{i'}^T (\R \mathbf{s}_{i} + \mathbf{t} - \mathbf{s}_{i'})]_\times \Q_{i'}^T \R \Q_i \nonumber \\
& = \Q_{i'}^T [\R \mathbf{s}_i + \mathbf{t} - \mathbf{s}_{i'}]_\times \R \Q_i \nonumber \\
& = \Q_{i'}^T (\R [\mathbf{s}_i]_\times + [\mathbf{t} - \mathbf{s}_{i'}]_\times \R) \Q_i.
\label{eq:essential_matrix}
\end{align}
To derivate the second and third equalities, a property that $[\R \mathbf{t}]_\times \R = \R [\mathbf{t}]_\times, \forall \R \in \text{SO}(3)$ is exploited. This formula changes $\E_k$ from being quadratic in $\R$ to linear in $\R$. 
It can be seen that Eqs.~\eqref{equ:sub1} and~\eqref{equ:sub2} is bilinear in the unknown $\R$ and $\mathbf{t}$. Notably, these two equations are homogeneous in $\R$ and are not homogeneous in $\mathbf{t}$. This inhomogeneity is the key that the scale of $\mathbf{t}$ can be recovered. 
In Section~\ref{sec:degenerate}, it shows that there are degenerate cases when the inhomogeneity disappears.

\subsection{Minimal Configurations}

\begin{table}[tbp]
	\caption{Excess constraints for different configurations. There are $n$~ACs across $m$~views.}
	\label{tab:configurations}
	\centering
	\begin{tabular}{|c|*{7}{c|}}
		\hline
		\backslashbox[3em]{$m$}{$n$}
		& $1$ & $2$ & $3$ & $4$ & $5$ & $6$ & $7$\\
		\hline
		$1$ & $-3$ & $-6$ & $-9$ & $-12$ & $-15$ & $-18$ & $-21$ \\\hline
		$2$ & $-3$ & $\mathbf{0}$ & $3$ & $6$ & $9$ & $12$ & $15$ \\\hline
		$3$ & $-3$ & $6$ & $15$ & $24$ & $33$ & $42$ & $51$ \\\hline
		$4$ & $-3$ & $12$ & $27$ & $42$ & $57$ & $72$ & $87$ \\\hline
	\end{tabular}
\end{table}

For a single camera, it is demonstrated that two ACs are sufficient to recover the relative pose~\cite{raposo2016theory}. Generally speaking, two ACs provide $6$ independent constraints, and there are 5DOF in the relative pose of a single camera. Thus there is one excess constraint. We can ignore excess constraints during the solver generation procedure.

For multi-camera systems, the minimal configurations are complicated. 
Given $n$ ACs across $m$ views, we will figure out all of the minimal configurations. Assume that one AC is captured by $m$-views of a multi-camera system. In this case, there are two kinds of unknowns. First, there are $5$ unknowns for each oriented point in 3D space corresponding to an AC, including $3$ for point position and $2$ for direction. Second, there are $6(m-1)$ unknowns for the pose of $m$ views considering the world reference can only be determined up to a rigid motion. In summary, there are $5n + 6(m-1)$ unknowns. 
Each AC introduces $2m$ knowns for point position and $4(m-1)$ independent parameters for affine transformations. In summary, there are $2mn+4(m-1)n$ knowns. 

A configuration belongs to a minimal problem if and only if the number of unknowns equals the number of knowns. 
The results of excess constraint, i.e., the number of knowns minus the number of unknowns, are shown in Table~\ref{tab:configurations}. Zero indicates a minimal configuration. Positive numbers and negative numbers indicate over-determined and under-determined configurations, respectively.
Since $m$ and $n$ are positive integers, $m = n = 2$ corresponds to the only minimal configuration. It should be mentioned that Table~\ref{tab:configurations} summarizes the results for general settings. It is not necessarily held for special cases such as the relative pose has a motion prior, or the ACs are under partial visibility in multiple views. 

There are many variants of minimal configurations. For example, 1AC+3PC across two views is also a minimal configuration. Another example is that one oriented point is captured by many perspective cameras in a view. It is redundant to enumerate all these variants. However, the solver generation procedure in this paper can be extended to these configurations in a straightforward way.

\subsection{Categories of Two ACs Across Two Views}
In the previous analysis, we draw a conclusion that two ACs across two views is a minimal configuration for a multi-camera system. 
By excluding one symmetry between individual perspective cameras and one symmetry between the two views, this configuration can be further classified into $9$ cases, see Fig.~\ref{fig:ac_type}. 
In the following, the AC type $(i, i')$ means the AC appears in the $i$-th camera in the first view and the $i'$-th camera in the second view.

Among the $9$ cases, the most common ones in practice are cases $6$ and $7$ for two-camera rigs. As shown in Fig.~\ref{fig:ac_type}, case $6$ uses inter-camera ACs, which is suitable for two-camera rigs with extensive overlapping of views. Case $7$ uses intra-camera ACs, which is suitable for two-camera rigs with non-overlapping or small-overlapping of views. 
Cases $4$ and $5$ are useful for the incremental structure-from-motion based on a novel structure-less camera resection~\cite{zheng2015structure,kasten2019resultant}, in which the collection of already reconstructed cameras are viewed as a generalized camera.
Cases $8$ and $9$ are degenerate since they can be viewed as that two ACs are captured by a single perspective camera from two views. 

\begin{figure}[tbp]
	\begin{center}
		\includegraphics[width=0.8\linewidth]{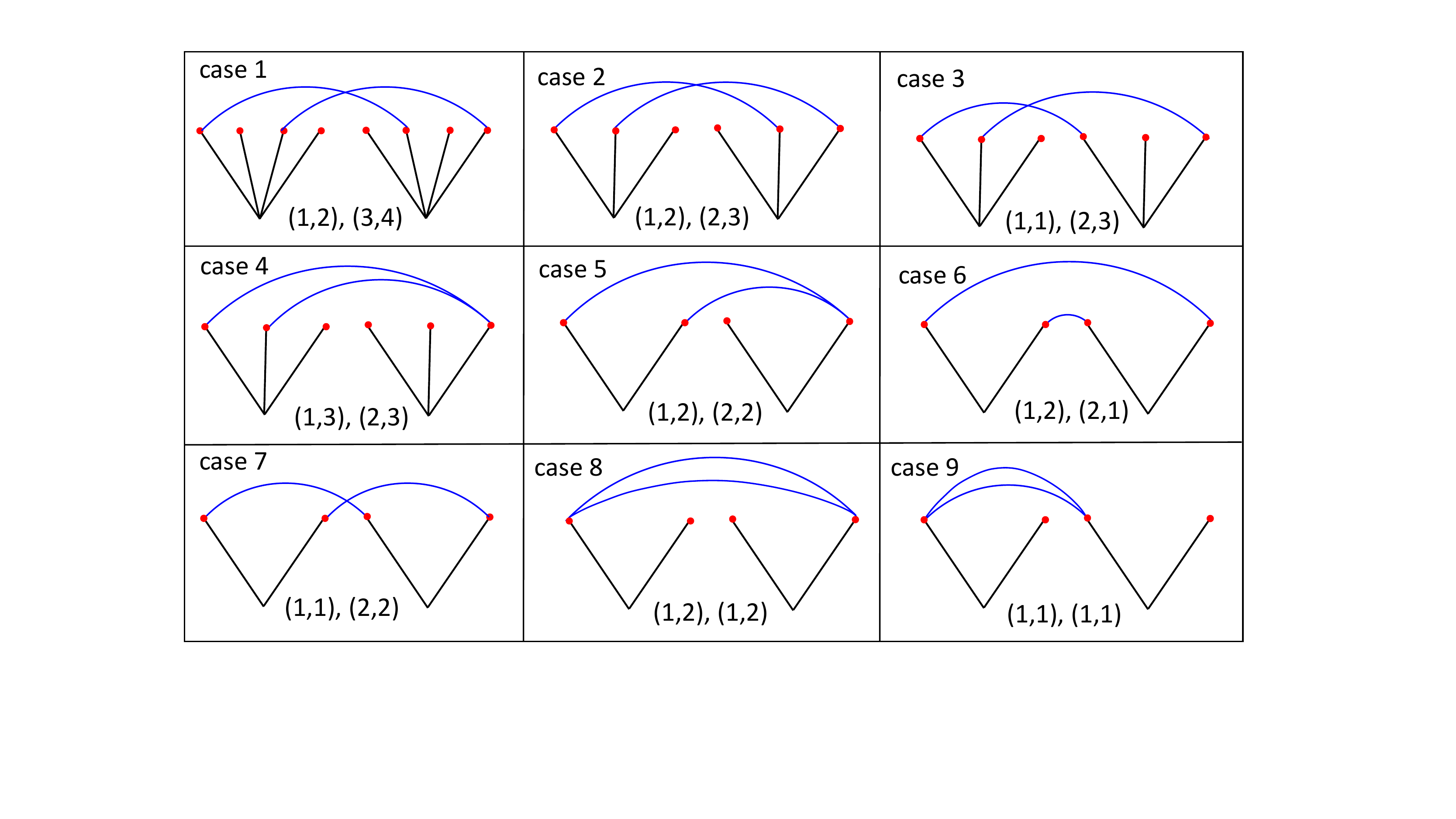}
	\end{center}
	\vspace{-0.1in}
	\caption{Nine types of affine correspondence. Red dots represent perspective cameras; blue arcs represent affine correspondences. A pair $(i, i')$ represents an AC that appears in the $i$-th camera in the first view and the $i'$-th camera in the second view. 
	}
	\label{fig:ac_type}
\end{figure}

\subsection{Degenerate Cases}
\label{sec:degenerate}

For multi-camera systems, the scale is observable due to $\E$ is inhomogeneous in translation $\mathbf{t}$. In some special configurations, $\E$ becomes homogeneous in translation~$\mathbf{t}$, which makes the scale unobservable.
We prove three degenerate cases for multi-camera systems. 

\begin{proposition}
	\label{theorem:inter_cam}
	For case~6 in Fig.~\ref{fig:ac_type}, when a multi-camera system undergoes pure translation and the translation direction is consistent with the baseline of two cameras, this configuration is degenerate. Specifically, the scale of translation cannot be recovered.
\end{proposition}
\begin{proof}
	Suppose the $k$-th AC appears in the $i$-th camera of view~1 and the $i'$-th camera of view~2. 
	Since the translation direction is parallel with baseline of two camera, translation satisfies $\mathbf{s}_i-\mathbf{s}_{i'} = a \mathbf{t}$, where $a$ is a unknown factor.
	In the case of pure translation,~\emph{i.e.}, ${\mathbf{R}=\mathbf{I}}$, 
	the essential matrix $\E_k$ in Eq.~\eqref{eq:essential_matrix} becomes
	\begin{align}
	\E_k &= \Q_{i'}^T ([\mathbf{t} + \mathbf{s}_i - \mathbf{s}_{i'}]_\times) \Q_i \nonumber \\
	&= (a+1) \Q_{i'}^T [\mathbf{t}]_\times \Q_i.
	\end{align}
	The essential matrix is homogeneous with $\mathbf{t}$. 
	For an arbitrary AC, the previous analysis holds. Thus the scale of translation cannot be recovered.
\end{proof}

\begin{figure}[tpb]
	\begin{center}
		\includegraphics[width=0.9\linewidth]{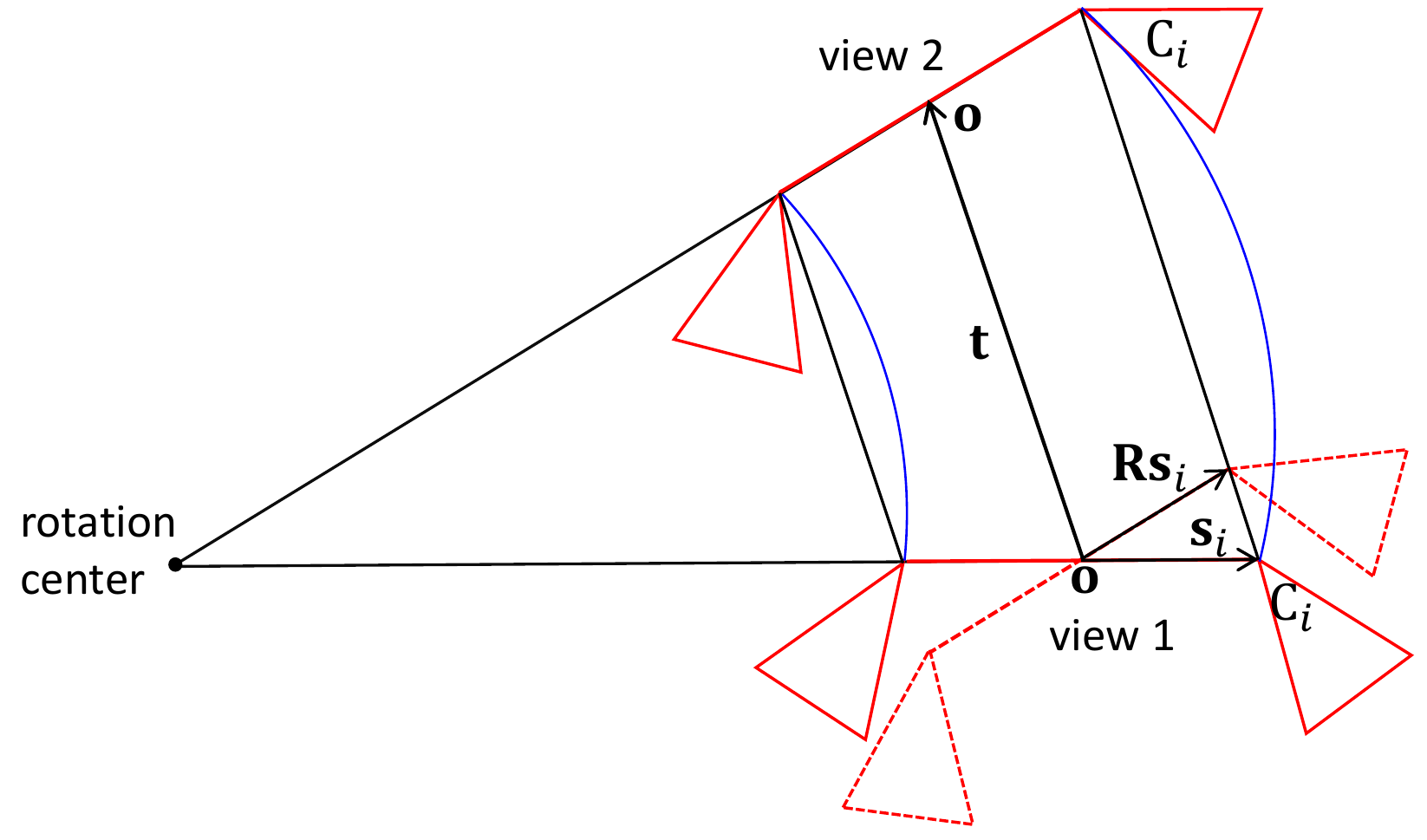}
	\end{center}
	\vspace{-0.1in}
	\caption{A degenerate case due to constant rotation rate for case~7. }
	\label{fig:Criticalmotion}
\end{figure}

\begin{proposition}
	\label{theorem:pure_translation}
	For case~7 in Fig.~\ref{fig:ac_type}, the motion of pure translation or constant rotation rate is degenerate. Specifically, the scale of translation cannot be recovered.
\end{proposition}

\begin{proof}
	Suppose $k$-th AC appears in the $i$-th camera of view~1 and the $i'$-th camera of view~2. 
	Since each AC is an intra-camera correspondence for case~7, we have $\mathbf{s}_i = \mathbf{s}_{i'}$ and $\Q_i = \Q_{i'}$. 
	
	For pure translation case, i.e., $\R = \mathbf{I}$, essential matrix in in Eq.~\eqref{eq:essential_matrix} becomes $\E_k = \Q_{i}^T [\mathbf{t}]_\times \Q_i$, which is homogeneous with $\mathbf{t}$. 
	For the constant rotation rate case, the proof is inspired by~\cite{clipp2008robust} for PC-based solvers. 
	In Fig.~\ref{fig:Criticalmotion}, both cameras move along with concentric circles. The two blue arcs are the camera trajectories moving along concentric circles. The motion of the multi-camera system is equivalent to a pure rotation $\R$ at first, then proceed by a pure translation $\mathbf{t}$. The dotted multi-camera system is the result of view~1 under pure rotation. Point $\mathbf{o}$ is the origin of the multi-camera system's reference. Without loss of generality, we assume that point $\mathbf{o}$ lies on the two-camera rig's baseline.
	Take single camera $C_i$ for an example, the pure rotation induced translation ${{\mathbf{R}}{\mathbf{s}_{i}}-{\mathbf{s}_{i}}}$ is aligned with the pure translation ${\mathbf{t}}$. Denote $\lambda{\mathbf{t}} = {{\mathbf{R}}{\mathbf{s}_i}-{\mathbf{s}_i}}$ and substitute it to Eq.~\eqref{eq:essential_matrix}, it can be verified that $\lambda{\mathbf{t}}$ invariably satisfies Eqs.~\eqref{equ:sub1} and~\eqref{equ:sub2}. 
\end{proof}

\begin{proposition}
	\label{theorem:degenerate_case}
	Case~8 and case~9 in Fig.~\ref{fig:ac_type} are degenerate. Specifically, the translation cannot be recovered.
\end{proposition}
\begin{proof}
	These two cases can be viewed as that two ACs are captured by a single camera from two views. The relative rotation and translation can be recovered by a minimal solver for a single camera, such as~\cite{barath2018efficient} or the solver proposed in this paper. However, the recovered translation has scale-ambiguity for a single camera. The relative pose between a multi-camera system is a composition of three transformations, including two extrinsic parameters and the relative pose between two views of the single perspective camera. Due to the scale-ambiguity between two views of the perspective camera, the translation between two views of the multi-camera system cannot be recovered. 
\end{proof}

\section{Relative Pose Recovery from Affine Correspondences}
\label{sec:complete_solution}

In this section, we propose a series of minimal solvers for all the cases in Fig.~\ref{fig:ac_type}. The proposed solvers form a complete solution to relative pose estimation from ACs.

First of all, we need to parametrize the relative pose. Rotation can be parameterized by Cayley, quaternions, Euler angles, direction cosine matrix (DCM), etc. Cayley and quaternion parameterizations have shown superiority in minimal problems~\cite{zhao2020minimal}. 
Rotation $\mathbf{R}$ using Cayley parameterization can be written as 
\begin{align}
&\mathbf{R}_{\text{cayl}} = \frac{1}{q_x^2+q_y^2+q_z^2+1} \ . \nonumber \\ 
&
\begin{bmatrix}
{1+q_x^2-q_y^2-q_z^2} &  2 q_x q_y -2 q_z & 2 q_x q_z + 2 q_y \\
2 q_x q_y+2 q_z & 1-q_x^2+q_y^2-q_z^2 & 2 q_y q_z - 2 q_x \\
2 q_x q_z - 2 q_y & 2 q_y q_z + 2 q_x & {1-q_x^2-q_y^2+q_z^2}
\end{bmatrix},	
\label{eq:R6dof1}
\end{align}
where $[1,q_x,q_y,q_z]^T$ is a homogeneous quaternion vector. Note that $180$-degree rotations are prohibited in Cayley parameterization, but this is a rare case for usual image pairs. In practice, it has been widely used in minimal problems~\cite{stewenius2005minimal,kneip2014efficient,zheng2015structure,zhao2020minimal}.
Rotation $\R$ using quaternion parameterization can be written as
\begin{align}
&\mathbf{R}_{\text{quat}} =  \nonumber \\ 
&
\begin{bmatrix}
q_w^2+q_x^2-q_y^2-q_z^2 &  2 q_x q_y -2 q_w q_z & 2 q_x q_z + 2 q_w q_y \\
2 q_x q_y+2 q_w q_z & q_w^2-q_x^2+q_y^2-q_z^2 & 2 q_y q_z - 2 q_w q_x \\
2 q_x q_z - 2 q_w q_y & 2 q_y q_z + 2 q_w q_x & q_w^2-q_x^2-q_y^2+q_z^2
\end{bmatrix},	
\label{eq:R6dof2}
\end{align}
where $[q_w^2,q_x,q_y,q_z]^T$ is a normalized quaternion vector satisfying
\begin{align}
q_w^2 + q_x^2 + q_y^2 + q_z ^2 = 1.
\end{align}
In the following, we take Cayley parameterization as an example. The solver generation procedure can be applied to quaternion parameterization straightforwardly.

The translation $\mathbf{t}$ can be written as 
\begin{align}
\mathbf{t} = \begin{bmatrix}
{t_x}& \
{t_y}& \
{t_z}
\end{bmatrix}^T.
\label{eq:T6dof1}
\end{align} 

\subsection{A Property of Rotation Matrices with Unknown Variables}

A rotation matrix is an orthogonal matrix. We find a property for this kind of matrices, which is useful for relative pose estimation.

\begin{theorem}
	Suppose a matrix $\Q_{3\times 3}$ satisfies
	\begin{align}
	\Q \Q^T = s^2 \mathbf{I}, \ s > 0,
	\label{equ:condition}
	\end{align}
	where $s$ is a polynomial of some variables.
	Suppose $\a_1^{(i)}$, $\a_2^{(i)}$, $\a_3^{(i)}$, $\b_1^{(i)}$, $\b_2^{(i)}$, and $\b_3^{(i)} \in \mathbb{R}^3$ are arbitrary non-zero vectors. 
	Then the polynomial of determinant
	\begin{align}
	\mathbf{N} = \begin{bmatrix}
	\left(\sum_{i=1}^{m_1} \a_1^{(i)} \times \Q \b_1^{(i)}\right)^T \\ \left(\sum_{i=1}^{m_2} \a_2^{(i)} \times \Q \b_2^{(i)}\right)^T \\ \left(\sum_{i=1}^{m_3} \a_3^{(i)} \times \Q \b_3^{(i)}\right)^T
	\end{bmatrix}
	\label{eq:ggform}
	\end{align}
	has a factor $s$, where $m_1, m_2, m_3$ are positive integers.
	\label{theorm:general_form}
\end{theorem}
The proof is provided in the supplementary material. 
$\mathbf{R}_{\text{quat}}$ satisfies the condition~\eqref{equ:condition} with $s = q_w^2 + q_x^2 + q_y^2 + q_z^2$.
When an equation system is homogeneous in $\R_{\text{cayl}}$, we can safely omit the denominator in $\R_{\text{cayl}}$. It can be verified 
that the matrix $(1+q_x^2+q_y^2+q_z^2) \R_{\text{cayl}}$ also satisfies the condition~\eqref{equ:condition} with $s = 1 + q_x^2 + q_y^2 + q_z^2$. 
In the following, we will use this theorem to factor out $s$ to simplify the equation system. It will generate more efficient solvers and sometimes avoid false roots.

In a few papers on relative pose estimation, a scale is factored out from the equation system for several specific parameterizations~\cite{sweeney2014solving,zhao2020minimal}. 
The Theorem~\ref{theorm:general_form} provides a rigorous theory for a general form for the first time. 
We provide a good practice according to this theorem. For $\mathbf{R}_{\text{quat}}$, there is an equivalent form of Eq.~\eqref{eq:R6dof2}: the diagonal elements can be replaced by $[1-2(q_y^2+q_z^2), 1-2(q_x^2+q_z^2), 1-2(q_x^2+q_y^2)]^T$. This form does not satisfy the condition in Theorem~\ref{theorm:general_form}. Thus the factor $(1+q_x^2+q_y^2+q_z^2)$ cannot be factored out. It will usually cause a larger elimination template than using the form of Eq.~\eqref{eq:R6dof2}. Using this practice, we obtain a new five-point solver for single cameras based on quaternion, which has a smaller elimination template than that proposed in~\cite{zhao2020minimal}.

\subsection{Equation System Construction for Single Cameras}

For the $k$-th affine correspondence, we obtain three polynomials for six unknowns $\{q_x, q_y, q_z, t_x, t_y, t_z\}$ from Eqs.~\eqref{eq:constraint_epipolar} and~\eqref{eq:constraint_affine} by substituting the essential matrix~\eqref{eq:essential} into them. After separating $q_x$, $q_y$, $q_z$ from $t_x$, $t_y$, $t_z$, we arrive at an equation system 	
\begin{align} 
	\frac{1}{q_x^2+q_y^2+q_z^2+1}\underbrace {\overline{\M}_k(q_x, q_y, q_z)}_{3\times 3}
	\begin{bmatrix}
	{{t}_x}\\
	{{t}_y}\\
	{{t}_z}
	\end{bmatrix} = \mathbf{0},
	\label{eq:equ_qxqyqz}
\end{align}
where the entries of $\overline{\M}_k$ are quadratic polynomials with three unknowns $q_x,q_y,q_z$. 
Eq.~\eqref{eq:equ_qxqyqz} imposes three independent constraints on six unknowns $\{q_x, q_y, q_z, t_x, t_y, t_z\}$. 
Given two ACs, we get an equation system of $6$ independent constraints in a similar form as Eq.~\eqref{eq:equ_qxqyqz}.	By ignoring the scale factor, these equations are stacked by
\begin{align} 
	\underbrace {\begin{bmatrix}
		\overline{\M}_1(q_x, q_y, q_z)\\
		\overline{\M}_2(q_x, q_y, q_z)
		\end{bmatrix}}_{\overline{\M}_{6\times 3}} \begin{bmatrix}
	{{{t}_x}}\\
	{{{t}_y}}\\
	{{{t}_z}}
	\end{bmatrix} = {\mathbf{0}}.
	\label{eq:scale_M_t}
\end{align}

We use the hidden variable technique in equation system construction. The technique has been widely used in algebraic geometry for the elimination of variables from a multivariate polynomial system~\cite{cox2006using}. 
It can be seen that $\overline{\M}$ has a null vector. Its rank should be two for non-degenerate cases. (The rank cannot be less than two. Otherwise, the translation vector cannot be recovered.) Thus, the determinants of all the $3\times3$ submatrices of $\overline{\M}$ should be zero.
There is a property for these submatrices.
\begin{theorem}
	\label{theorem:factor_mono_cam}
	Suppose $\mathbf{N}$ is an arbitrary $3\times 3$ submatrix of $\overline{\M}$, the polynomial $\det(\mathbf{N})$ has a factor $q_x^2 + q_y^2 + q_z^2 + 1$.
\end{theorem}
The proof is based on Theorem~\ref{theorm:general_form}, and it is provided in the supplementary material.
Based on Theorem~\ref{theorem:factor_mono_cam}, the equation system is 
\begin{align}
\quot(\det(\mathbf{N}), q_x^2+q_y^2+q_z^2+1) = 0, \nonumber \\
\mathbf{N} \in 3\times 3 \text{ submatrices of } \overline{\M}.
\label{eq:submatrix_3by3}
\end{align}
where $\quot(a, b)$ means quotient of $a$ divided by $b$, and $\det(\cdot)$ means the determinant operator.

There are two advantages to divide the determinants $\det(\mathbf{N})$ by $q_x^2+q_y^2+q_z^2+1$. 
First, the division eliminates the factor and avoids extraneous roots satisfying $q_x^2+q_y^2+q_z^2+1 = 0$. Though this equality cannot be satisfied in the real number field, it might be satisfied in the complex number field. Second, the order of the equation system can be reduced from $6$ to $4$.

There are $20$ equations of degree $4$ in Eq.~\eqref{eq:submatrix_3by3}. 
During solver generation, we ignore one excess constraint by removing the last row of $\overline{\M}$. Then there are $10$ equations of degree $4$. 
Once the rotation parameters $\{q_x, q_y, q_z\}$ have been obtained, the translation $[t_x, t_y, t_z]^T$ can be recovered by calculating the null space of $\overline{\M}$ up to a scale.  The excess constraint is used to select the correct solution as that in~\cite{guan2020relative}.

\subsection{Equation System Construction for Multi-Camera Systems}

For the $k$-th affine correspondence, we obtain three polynomials for six unknowns $\{q_x, q_y, q_z, t_x, t_y, t_z\}$ from Eqs.~\eqref{equ:sub1} and~\eqref{equ:sub2} by substituting the essential matrix~\eqref{eq:essential_matrix} into them. After separating $q_x$, $q_y$, $q_z$ from $t_x$, $t_y$, $t_z$, we arrive at an equation system 	
\begin{align} 
	\frac{1}{q_x^2+q_y^2+q_z^2+1}\underbrace {
	\M_k(q_x, q_y, q_z)}_{3\times 4}
	\begin{bmatrix}
	{{{t}_x}}\\
	{{{t}_y}}\\
	{{{t}_z}}\\
	1
	\end{bmatrix} = {\mathbf{0}},
	\label{eq:equ_qxqyqz1}
\end{align}
where the entries of $\M_k$ are quadratic polynomials with three unknowns $q_x,q_y,q_z$. 

Eq.~\eqref{eq:equ_qxqyqz1} imposes three independent constraints on six unknowns $\{q_x, q_y, q_z, t_x, t_y, t_z\}$. 
Given two ACs, we get an equation system of $6$ independent constraints in a similar form as Eq.~\eqref{eq:equ_qxqyqz1}. 
By ignoring the scale factor, these equations are stacked by
\begin{align} 
	\underbrace {\begin{bmatrix}
		\M_1(q_x, q_y, q_z)\\
		\M_2(q_x, q_y, q_z)
		\end{bmatrix}}_{\M_{6\times 4}} \begin{bmatrix}
	{{{t}_x}}\\
	{{{t}_y}}\\
	{{{t}_z}}\\
	1
	\end{bmatrix} = {\mathbf{0}}.
	\label{eq:scale_M_t_1}
\end{align}
It can be seen that ${{\M}}$ has a null vector. Its rank should be three for non-degenerate cases. (The rank cannot be less than three. Otherwise, the translation vector cannot be recovered.) 
Thus, the determinants of all the $4\times4$ submatrices of ${\M}$ should be zero. 
In addition, there is additional implicit constraints in this case. In the following, we prove that the rank of $(\M_k)_{(1:3,1:3)}$, $\forall k \in \{1, 2\}$ is $2$. 
\begin{theorem}
	\label{theorem:extra_constraint}
	For non-degenerate cases, $\rank(\mathbf{N}_k) = 2$, $\forall k$, where $\mathbf{N}_k = (\M_k)_{(1:3,1:3)}$.
\end{theorem}
\begin{proof}
	Suppose the $k$-th AC appears in the $i$-th camera of view~1 and the $i'$-th camera of view~2. 
	
	First we prove that $\rank(\mathbf{N}_k) \le 2$. To achieve this goal, we need to prove that the null space of $\mathbf{N}_k$ is not empty.
	According to Eq.~\eqref{eq:essential_matrix}, essential matrix $\E_k$ is
	\begin{align}
	\E_k = \Q_{i'}^T [\mathbf{t} + \R \mathbf{s}_i - \mathbf{s}_{i'}]_\times \R \Q_i.
	\end{align}
	Denote 
	\begin{align}
	\bar{\mathbf{t}} \triangleq \mathbf{t} + \R \mathbf{s}_i - \mathbf{s}_{i'},
	\end{align}
	then we have
	\begin{align}
	\E_k = \Q_{i'}^T [\bar{\mathbf{t}}]_\times \R \Q_i.
	\label{eq:new_essential_matrix}
	\end{align}
	Substituting Eq.~\eqref{eq:new_essential_matrix} into Eqs.~\eqref{equ:sub1} and~\eqref{equ:sub2}, we obtain three equations. Each monomial in these three equations is linear with one entry of vector $\bar{\mathbf{t}}$, and there is no constant term. Thus these three equations can be formulated as
	\begin{align}
	&\frac{1}{q_x^2+q_y^2+q_z^2+1} \mathbf{A}_{k} \bar{\mathbf{t}} = \mathbf{0},
	\label{eq:null_space} \\
	\Rightarrow & \mathbf{A}_k (\mathbf{t} + \R \mathbf{s}_i - \mathbf{s}_{i'}) = \mathbf{0}, \\
	\Rightarrow & 
	\begin{bmatrix}
	\mathbf{A}_k & \mathbf{A}_k(\R \mathbf{s}_i - \mathbf{s}_{i'})
	\end{bmatrix}
	\begin{bmatrix}
	\mathbf{t} \\ 1
	\end{bmatrix}
	= \mathbf{0}.
	\label{eq:equ_qxqyqz1_new_form}
	\end{align}
	By comparing the construction procedure of Eq.~\eqref{eq:equ_qxqyqz1} and Eq.~\eqref{eq:equ_qxqyqz1_new_form}, we can see that
	\begin{align}
	\mathbf{A}_k = (\mathbf{M}_k)_{(1:3,1:3)} = \mathbf{N}_k.
	\end{align}
	Substituting this equation into Eq.~\eqref{eq:null_space}, we can see that the null space of $\mathbf{N}_k$ is not empty. 
	
	Next we prove that $\rank(\mathbf{N}_k) \ge 2$. We achieve this goal using proof by contradiction. If $\rank(\mathbf{N}_k) \le 1$, then $\rank(\M_k) \le 2$ considering that $\M_k$ has one more column than $\mathbf{N}_k$. This means the $k$-th AC provides at most two independent constraints for the relative pose. This cannot be true for non-degenerate cases, so the assumption that $\rank(\mathbf{N}_k) \le 1$ is wrong.
\end{proof}

Similar to the single camera case, there is also a property for some submatrices of $\M$.
\begin{theorem}
	\label{theorem:factor_gcam}
	Suppose $\mathbf{N}$ is an arbitrary $4\times 4$ submatrix of $\M$ or an arbitrary $3\times 3$ submatrix of the first three columns of $\M$, the polynomial $\det(\mathbf{N})$ has a factor $q_x^2 + q_y^2 + q_z^2 + 1$. 
\end{theorem}
The proof is based on Theorem~\ref{theorm:general_form}, and it is provided in the supplementary material.
Based on Proposition~\ref{theorem:factor_gcam}, we divide the determinants by factor $q_x^2+q_y^2+q_z^2+1$, and the constraints become 
\begin{align}
\mathcal{E}_1 \triangleq \{ \quot(\det(\mathbf{N}), q_x^2+q_y^2+q_z^2+1) = 0 \ | \nonumber \\
\mathbf{N} \in 4\times 4 \text{ submatrices of } \M\}.
\label{eq:submatrix_4by4}
\end{align}
and
\begin{align}
\mathcal{E}_2 \triangleq \{\quot(\det(\mathbf{N}), q_x^2+q_y^2+q_z^2+1) = 0 \ | \nonumber \\
\mathbf{N} \in \{(\M_k)_{(1:3,1:3)}\}_{k=1,2} \}
\label{eq:submatrix_3by3_extra}
\end{align}
There are $15$ equations of degree $6$ and two equations of degree $4$ in $\mathcal{E}_1$ and $\mathcal{E}_2$, respectively.

Once the rotation
parameters $\{q_x, q_y, q_z \}$ have been obtained, the translation $[t_x, t_y, t_z]^T$ can be recovered by first calculating a vector in the null space of $\M$, and then normalizing the vector by dividing its last entry.

{\bf Remark}: We provide an explanation to Proposition~\ref{theorem:extra_constraint}. Vector $\bar{\mathbf{t}}$ is the translation vector between the $i$-th camera in view 1 and the $i'$-th camera in view 2, which is expressed in view 1.
If an AC is captured by a single camera, the equation system is homogeneous in $\mathbf{t}$. In contrast, if an AC is captured by a multi-camera system, the equation system is inhomogeneous in $\mathbf{t}$. 
When an AC is captured by a multi-camera system, it still satisfies the two-view geometry for a single camera by using proper translation.

\subsection{Instantiation in A Finite Prime Field}

Equation system~\eqref{eq:submatrix_3by3} and equation system~\eqref{eq:submatrix_4by4}\eqref{eq:submatrix_3by3_extra} can be solved by Gr{\"o}bner basis method~\cite{cox2013ideals}, which is a general method to solve polynomial equation systems.
Automatic solver generators~\cite{kukelova2008automatic,larsson2017efficient} can be used to construct solvers based on the Gr{\"o}bner basis method.
The most critical and challenging step in automatic solver generation is constructing a random instance of the original equation system in a finite prime field $\mathbb{Z}_p$~\cite{lidl1997finite}. This step aims to keep numerical stability and avoid large number arithmetic during the calculation of Gr{\"o}bner basis.
When constructing the random instance, the relations between the coefficients should be appropriately preserved. Otherwise, setting random values of the equation coefficients would destroy the latent relations and might result in a different problem without any solution. 
The technique of instantiation has been used in many minimal problems~\cite{stewenius2005solutions,bujnak2012algebraic,zhao2020minimal,pritts2020minimal}.

In our problem, the coefficients of the equation systems are not fully independent since they are determined by certain latent relations. During the instantiation of our problem, we take advantage of basic operations in previous literature~\cite{bujnak2012algebraic} and develop new operations in a finite field. 
In our problem, some new operations in $\mathbb{Z}_p$ should be defined. 
First, random oriented points (or infinitesimal patches equivalently) in 3D space should be defined appropriately. Each patch is defined by a point $\p$ and a unit normal $\n$ in 3D space. The equation of the plane is $\n^T(\y-\p) = 0$, where $\y$ is an arbitrary point in the plane. Since there is no square root for each number in finite prime fields, we need to try several times to generate a random unit normal.  
Second, the signed distance $d_0$ from a point $\y_0$ to previously defined plane is calculated by $d_0  = \n^T(\y_0 - \p)$ in $\mathbb{Z}_p$. 
Third, when a plane is captured by two views of a single camera, the homography is calculated by $\mathbf{H} = \R' + \frac{1}{d} \mathbf{t}' \mathbf{n}^T$, where $[\R', \mathbf{t}']$ is the relative pose from the first view to the second view, $\n$ is the unit normal of the plane expressed in the first view, and $d$ is the signed distance from the optical center of the first view to the plane. Finally, the affine transformation $\A$ can be calculated given the homograpy $\mathbf{H}$ and image coordinates of the point correspondence. The formula in the real number field can be found in~\cite{barath2018efficient}. Since there is only addition, subtraction, multiplication, and division in this formula, it can be directly transferred to finite prime field $\mathbb{Z}_p$ using the same form.

\subsection{A Series of Solvers}

\begin{table}[tbp]
	\centering
	\caption{Minimal solvers for 6DOF relative pose estimation of multi-camera systems. cayl: Cayley; quat: quaternion; inter: inter-camera correspondence (case~$6$); intra: intra-camera correspondence (case~$7$). }
	\label{tab:complete_solution}
	\setlength{\tabcolsep}{1.5pt}{
	\begin{tabular}{|l|c|c|c|c|c|c|c|c|c|c|} 
		\hline
		\multirow{2}{*}{\centering configuration} &  \multicolumn{3}{c|}{equations $\mathcal{E}_1$} &  \multicolumn{3}{c|}{equations $\mathcal{E}_1+\mathcal{E}_2$}  \\ 
		\cline{2-7} 
		&   \#sym &  \#sol  & template   &  \#sym &  \#sol    &   template  \\ 
		\hline
		2ac+cayl+(case $1$-$5$)  & 0 & $64$ &  $99\times 163$ & 0 & $48$ & $72\times 120$ \\ \hline
		2ac+cayl+inter(case $6$) & 0 & $56$ & $56\times 120$ & 0 & $48$ & $64\times 120$  \\ \hline
		2ac+cayl+intra(case $7$) & 0 & $1$-dim & $-$ & 0 & $48$ & $72\times 120$   \\ \hline  \hline
		2ac+quat+(case $1$-$5$)  & 1 & $128$ &  $342\times 406$ & 1 & $96$ & $152\times 200$ \\ \hline
		2ac+quat+inter(case $6$) & 1 & $112$ & $178\times 243$ & 1 & $96$ & $152\times 200$  \\ \hline
		2ac+quat+intra(case $7$) & 1 & $1$-dim & $-$ & 1 & $96$ & $152\times 200$ \\ \hline
	\end{tabular}
    }
\end{table}

The solver generator of Larsson et al.~\cite{larsson2017efficient} was used to find a series of solvers for different cases. \texttt{Macaulay~2}~\cite{grayson2002macaulay} is used to calculate Gr{\"o}bner basis. Both the Cayley and quaternion parameterizations are exploited.

For a single camera and Cayley parameterization, the solver has $20$~solutions, and the elimination template is $36\times 56$. 
For quaternion parameterization, the solver has $40$~solutions with one symmetry, and the elimination template is $60\times 80$.  

For a multi-camera system, the statistics of the resulted solvers are shown in Table~\ref{tab:complete_solution}. 
\texttt{\#sym} represents the number of symmetries, and \texttt{\#sol} represents the number of solutions. \texttt{$1$-dim} represents one dimensional extraneous roots.
We have the following observations.
(1)~Cases $1\sim 5$ have the same solver. These cases can be viewed as one category. 
(2)~If only $\mathcal{E}_1$ is used, cases $1\sim 5$ maximally have $64$ complex solutions. 
Case $6$ has $56$ complex solutions. However, case $7$ has one-dimensional families of extraneous roots.
(3)~If both $\mathcal{E}_1$ and $\mathcal{E}_2$ are used, cases $1\sim 7$ have $48$ complex solutions. (4) Using equations from $\mathcal{E}_1$ + $\mathcal{E}_2$ results in smaller eliminate templates than using $\mathcal{E}_1$ only; using Cayley parameterization results in smaller eliminate templates than using quaternion parameterization. So the Cayley parameterization is preferred. (5) For quaternion parameterization, we also tested the method without factoring out the factor $q_x^2 + q_y^2 + q_z^2 + q_w^2$. It results in larger eliminate templates, demonstrating the effectiveness of factoring out the factor.

As shown in Proposition~\ref{theorem:degenerate_case}, cases~$8$ and $9$ are degenerate. In these two cases, the rotation can be uniquely recovered, but the translation cannot be recovered. The proof of Proposition~\ref{theorem:degenerate_case} provides a method of recovering the rotation. The core component is the relative pose for a single camera. We use our minimal solver to estimate the relative pose of a single camera.

At first glance, using solvers resulted from equations $\mathcal{E}_1+\mathcal{E}_2$ is the first choice since they have a small number of solutions than using $\mathcal{E}_1$ only. However, we found that solvers from $\mathcal{E}_1$ have better numerical stability for case~$6$. 
The phenomenon that a larger number of basis than minimal requirement might have better numerical stability has been observed in previous literature~\cite{byrod2009fast,larsson2018beyond}. 
An empirical comparison of numerical stability is shown in the experiments.

\section{Relative Pose Recovery from Point Correspondences}
\label{sec:pt_pose}

In~\cite{stewenius2005solutions}, a seminal six-point method is proposed for estimating the relative pose of generalized cameras. In this method, a generalized camera is formed by abstracting landmark observations into spatial rays. In practice, the most common generalized camera is the multi-camera system. We propose a new six-point method for multi-camera systems. Our method is based on the same framework for generating AC-based solvers, demonstrating its versatility.

\subsection{Equation System Construction and Solving}
For the relative pose estimation from PCs, the equation system construction and solving is essentially the same as the case of AC in Sections~\ref{sec:minimal_config} and~\ref{sec:complete_solution}. 
One PC in a multi-camera system relates two perspective cameras across two views. Denote the $k$-th PC as $(\x_k, \x'_k, i_k, i'_k)$. It represents that a point is captured by the $i_k$-th camera in the first view, and its homogeneous coordinate in the normalized image plane is $\x_k$. It is also captured by the $i'_k$-th camera in the second view, and its homogeneous coordinate is $\x'_k$. 
To simplify the notation, we omit the subscript $k$ of camera indices $i$ and $i'$ in the following text.
The equation introduced by the $k$-th PC is same as Eq.~\eqref{equ:sub1}, and the essential matrix $\E_k$ is determined by Eq.~\eqref{eq:essential_matrix}. We re-write them as below
\begin{align}
\x_k'^T \E_k \x_k = 0,
\label{eq:essential_mat_pc}
\end{align}
where
\begin{align}
\E_k = \Q_{i'}^T (\R [\mathbf{s}_i]_\times + [\mathbf{t} - \mathbf{s}_{i'}]_\times \R) \Q_i.
\label{eq:essential_ek_pc}
\end{align}

It is well-known that 6PC is a minimal configuration for relative pose estimation of generalized cameras. Using Cayley parameterization, there are 6 equations of the form~\eqref{eq:essential_mat_pc}. By using Cayley parametrization, these equations can be reformulated as
\begin{align} 
\underbrace {
	\widehat{\M}(q_x, q_y, q_z)}_{6\times 4}
\begin{bmatrix}
{{{t}_x}}\\
{{{t}_y}}\\
{{{t}_z}}\\
1
\end{bmatrix} = {\mathbf{0}},
\label{eq:equ_qxqyqz1_pc}
\end{align}
where the entries of $\widehat{\M}$ are quadratic polynomials with three unknowns $q_x,q_y,q_z$. 
The $i$-th row corresponds to the constraint of $i$-th PC.
It can be seen that $\widehat{\M}$ has a null vector. 
Thus, the determinants of all the $4\times4$ submatrices of $\overline{\M}$ should be zero.

For a family of PCs which relates the same perspective cameras across two views, there is one property. 
\begin{theorem}
	\label{theorem:extra_constraint_pc}
	For non-degenerate cases, $\rank(\mathbf{N}) = 2$, $\forall \mathbf{N} \in \mathcal{S}$. By Matlab syntax, $\mathcal{S}$ is a set whose elements satisfying $\mathbf{N} = \widehat{\M}([k_1,k_2,k_3], 1:3)$, where $k_1$-th, $k_2$-th, and $k_3$-th PCs are captured by the same perspective cameras across two views and $k_1 < k_2 < k_3$. 
\end{theorem}
The proof of is provided in the supplementary material. Again, a factor $ q_x^2+q_y^2+q_z^2+1$ can be factored out by applying Theorem~\ref{theorm:general_form}. 
In summary, the constraints are 
\begin{align}
\widehat{\mathcal{E}}_1 \triangleq \{ \quot(\det(\mathbf{N}), q_x^2+q_y^2+q_z^2+1) = 0 \ | \nonumber \\
\mathbf{N} \in 4\times 4 \text{ submatrices of } \widehat{\M}\}.
\label{eq:submatrix_4by4_pc}
\end{align}
and 
\begin{align}
\widehat{\mathcal{E}}_2 \triangleq \{ \quot(\det(\mathbf{N}), q_x^2+q_y^2+q_z^2+1) = 0 \ | 
\mathbf{N} \in \mathcal{S}\}.
\label{eq:submatrix_3by3_pc}
\end{align}
There are $15$ equations of degree $6$ in $\widehat{\mathcal{E}}_1$. 
The number of equations in $\widehat{\mathcal{E}}_2$ varies for different PC configurations. 
$\widehat{\mathcal{E}}_2$ might be empty set for certain configurations. 

\begin{figure}[tpb]
	\begin{center}
		\includegraphics[width=0.95\linewidth]{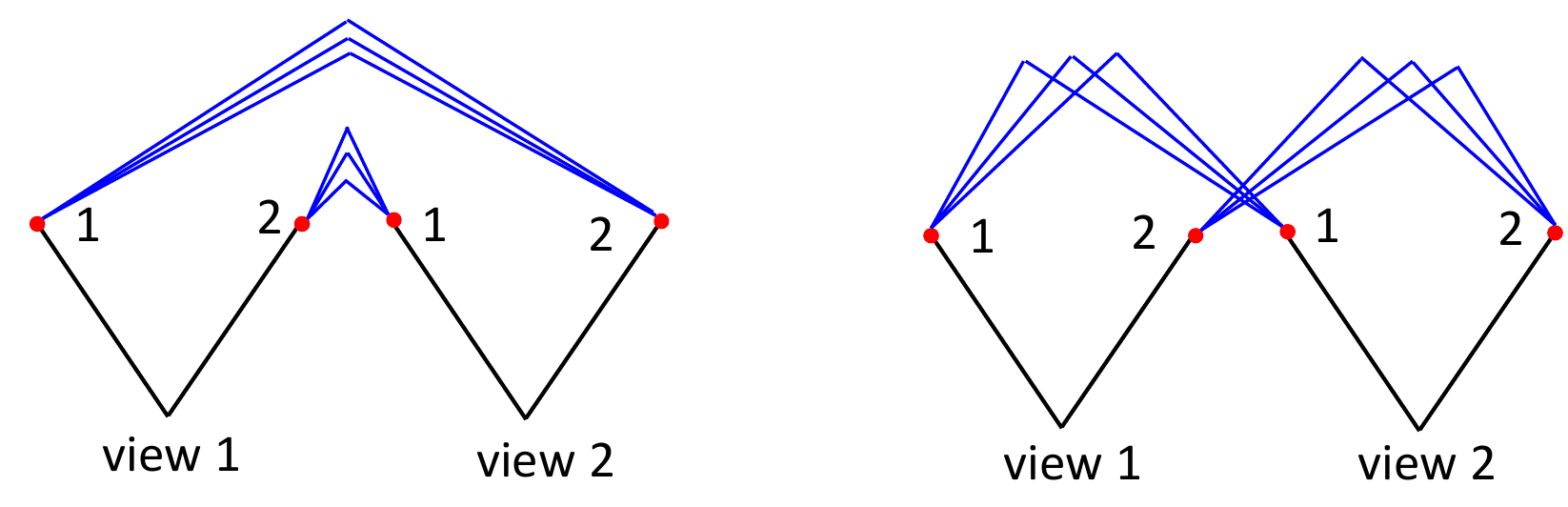}
	\end{center}
	\vspace{-0.1in}
	\caption{Relative pose estimation from PCs for a multi-camera system. Six points are captured by two views of a two-camera rig. We aim to recover the 6DOF relative pose using $6$~PCs. Left: inter-camera PCs; Right: intra-camera PCs.}
	\label{fig:teaser_pc_pose}
\end{figure}

\begin{table}[tbp]
	\centering
	\caption{Minimal solvers for 6DOF relative pose estimation of multi-camera systems. cayl: Cayley; quat: quaternion; inter: inter-camera correspondence; intra: intra-camera correspondence.}
	\label{tab:complete_solution_pc}
	\setlength{\tabcolsep}{2pt}{
		\begin{tabular}{|l|c|c|c|c|c|c|c|c|c|c|} 
			\hline
			\multirow{2}{*}{\centering configuration} &  \multicolumn{3}{c|}{equations $\hat{\mathcal{E}}_1$} &  \multicolumn{3}{c|}{equations $\hat{\mathcal{E}}_1+\hat{\mathcal{E}}_2$}  \\ 
			\cline{2-7} 
			&   \#sym &  \#sol  & template   &  \#sym &  \#sol    &   template  \\ 
			\hline
			6pt+cayl+inter & 0 & $56$ & $56\times 120$ & 0 & $48$ & $64\times 120$  \\ \hline
			6pt+cayl+intra & 0 & $1$-dim & $-$ & 0 & $48$ & $72\times 120$  \\ \hline  \hline
			6pt+quat+inter & 1 & $112$ & $174\times 243$  & 1 & $96$ &  $152\times 200$ \\ \hline
			6pt+quat+intra & 1 & $1$-dim & $-$ & 1 & $96$ & $152\times 200$  \\ \hline
		\end{tabular}
	}
\end{table}

We consider $2$ configurations for two-camera rigs. They are denoted as \texttt{6pt+inter} and \texttt{6pt+intra}, see Fig.~\ref{fig:teaser_pc_pose}. 
Both of them have two constraints in $\mathcal{E}_2$. 
The statistics of the resulted solvers are shown in Table~\ref{tab:complete_solution_pc}. For \texttt{6pt+quat+inter} configuration, an inequality $q_w \neq 0$ should be explicitly considered when using $\hat{\mathcal{E}}_1$ only. Otherwise, there is one-dimensional extraneous roots. We consider this inequality by the saturation method~\cite{larsson2017polynomial}, and obtain $112$ solutions with one symmetry.

\subsection{Relationship of AC-based and PC-based Solvers} 
When a plane is viewed in a pair of images, it is well-known that a homography relates the images of the plane~\cite{hartley2003multiple}. This is held for perspective cameras. For affine cameras, the homography can be simplified as an affine transformation. 
The key requirement for an affinity is that the imaging rays in each view are parallel, i.e., an orthogonal projection occurs. For perspective cameras, the affinity cannot be strictly held for image regions. However, local affinity is still satisfied for infinite-small neighborhoods of a point correspondence.  
An affine transformation is the first-order Taylor approximation of the related homography, i.e., it is tangent to the homography at the feature position~\cite{koser2008conjugate,bentolila2014conic}.

It was proved that one AC imposes three linear dependent constraints on relative pose~\cite{bentolila2014conic,raposo2016theory,barath2018efficient,eichhardt2018affine}. When constructing minimal configurations, 1AC can be roughly viewed as 3PCs. As a result, AC-based minimal solvers typically trisecting the number of minimum samples compared to PC-based counterparts. Take configurations \texttt{2ac+cayl+inter} and \texttt{6pt+cayl+inter} in this paper for an example, both of their minimal solvers have $48$ solutions. When using the same solver generator~\cite{larsson2017efficient}, the sizes of their elimination templates are equal. Despite their close relationship, 1AC is not identical to 3 PCs~\cite{bentolila2014conic}.

There arises a question. Is it possible to hallucinate three PCs from one AC, and use PC-based solvers for AC observations? Since the affinity holds for infinite-small regions only, we cannot exactly hallucinate points even given a noise-free affinity. In addition, there is a trade-off during the hallucination. On the one hand, the inter-distance of the hallucinated points should be small enough such that the approximation error is not large. On the other hand, the inter-distance should be large enough to avoid numerical instability of near-degeneration. In practice, there are a few methods to hallucinate 3 PCs from 1 AC~\cite{Chum2003epipolar,riggi2006fundamental}. It should be aware that hallucinated points will inevitably have approximation errors.

\section{Experiments}
\label{sec:experiment}

In this section, we conduct extensive experiments on synthetic and real-world data to evaluate the performance of the proposed solvers.  
For multi-camera systems, the proposed solvers are all based on Cayley parameterization.
For AC-based solvers, we focus on inter-camera configuration (case~6) and intra-camera configuration (case~7). The solvers are referred to as \texttt{2AC-inter} and \texttt{2AC-intra} methods for inter-camera and intra-camera ACs, respectively. 
Sometimes we need to further distinguish two solvers for \texttt{2AC-inter}. We use \texttt{2AC-inter-56} and \texttt{2AC-inter-48} to present solvers resulted from $\mathcal{E}_1$ and $\mathcal{E}_1 + \mathcal{E}_2$, respectively.
For PC-based solvers, we focus on inter-camera and intra-camera configurations. Their solvers are referred to as \texttt{6pt-inter} and \texttt{6pt-intra} methods, respectively. 
Sometimes we need to further distinguish two solvers for \texttt{6pt-inter}. We use \texttt{6pt-inter-56} and \texttt{6pt-inter-48} to present solvers resulted from $\widehat{\mathcal{E}}_1$ and $\widehat{\mathcal{E}}_1 + \widehat{\mathcal{E}}_2$, respectively.
The proposed solvers are compared with state-of-the-art solvers including  \texttt{17pt-Li}~\cite{li2008linear}, \texttt{8pt-Kneip}~\cite{kneip2014efficient}, and \texttt{6pt-Stew{\'e}nius}~\cite{stewenius2005solutions}. 
All the solvers are implemented in C++. The codes of comparison methods are provided by the OpenGV library~\cite{kneip2014opengv}.

In the real-world experiments, all the solvers are integrated into the RANSAC framework~\cite{fischler1981random} to reject outliers. The relative pose which produces the largest number of inliers is chosen. By following the default parameters in OpenGV~\cite{kneip2014opengv}, the confidence of RANSAC is $0.99$, and an inlier threshold angle is $0.1^\circ$. We demonstrate the feasibility of our methods on the \texttt{KITTI} dataset~\cite{geiger2013vision}.   

The rotation error is computed as the angular difference between the ground truth rotation and the estimated rotation: ${\varepsilon_{\mathbf{R}}} = \arccos ((\trace({\mathbf{R}_{\text{gt}}}{{\mathbf{R}^T}}) - 1)/2)$, where $\mathbf{R}_{\text{gt}}$ and ${\mathbf{R}}$ are the ground truth and estimated rotation matrices, respectively. 
By following the definition in~\cite{quan1999linear,lee2014relative}, the translation error is defined as ${\varepsilon_{\mathbf{t}}} = 2\left\| {{\mathbf{t}_{\text{gt}}}}-{\mathbf{t}}\right\|/(\left\| {\mathbf{t}_{\text{gt}}} \right\| + \left\| {{\mathbf{t}}} \right\|)$, where $\mathbf{t}_{\text{gt}}$ and ${\mathbf{t}}$ are the ground truth and estimated translations. The translation direction error is defined by $\varepsilon_{\mathbf{t},\text{dir}} = \arccos(\mathbf{t}_{\text{gt}}^T \mathbf{t} / (\|\mathbf{t}_{\text{gt}}\| \cdot \|\mathbf{t}\|) )$.

\subsection{Efficiency and Numerical Stability}

\begin{table*}[tbp]
	\caption{Runtime comparison of solvers for multi-camera systems (unit:~$\mu s$).}
	\begin{center}
	\setlength{\tabcolsep}{3pt}{
	{
		\begin{tabular}{lccccccccc}
		\toprule
		method & {17pt-Li}~\cite{li2008linear} & {8pt-Kneip}~\cite{kneip2014efficient} &  {6pt-Stew{\'e}nius}~\cite{stewenius2005solutions} & {6pt-inter-56} & {6pt-inter-48} & {6pt-intra} & {2AC-inter-56} & {2AC-inter-48} & {2AC-intra} \\
		\midrule
		mean time & 43.3 & 102.0 & 3275.4 & 1629.3& 1416.4&  1410.5 & 1669.5 & 1451.7 & 1491.8 \\
		\bottomrule
		\end{tabular}}}
	\end{center}
	\label{tab:runtime2}
\end{table*}

The runtimes of our solvers and the comparative solvers are evaluated using an Intel(R) Core(TM) i7-7800X 3.50GHz. 
Table~\ref{tab:runtime2} shows the average runtime of the solvers over $10,000$ runs for multi-camera systems. 
The proposed solvers take $1.4 \sim 1.7$ milliseconds. Among the minimal solvers, all the proposed solvers are more efficient than another minimal solver \texttt{6pt-Stew{\'e}nius}. 
Since \texttt{17pt-Li} is a linear solver, it is the most efficient.
As shown later, the proposed solvers need fewer correspondences than comparison methods and thus have better overall efficiency when integrating them into RANSAC.

\begin{figure}[tbp]
	\begin{center}
	\includegraphics[width=0.60\linewidth]{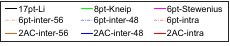}\\
	\vspace{-5pt} 
	\subfigure[rotation error $\varepsilon_{\R,\text{chordal}}$]
	{
		\includegraphics[width=0.47\linewidth]{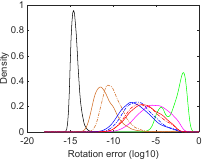}
	}
	\subfigure[translation error $\varepsilon_{\mathbf{t}}$]
	{
		\includegraphics[width=0.47\linewidth]{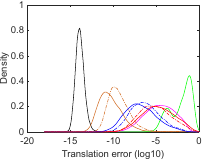}
	}
	\end{center}
	\vspace{-0.1in}
	\caption{Probability density functions over pose estimation errors on noise-free observations for multi-camera systems. The horizontal axis represents the $\log_{10}$ errors, and the vertical axis represents the density.}
	\label{fig:Numerical}
\end{figure} 

Figure~\ref{fig:Numerical} reports the numerical stability of the solvers on noise-free observations\footnote{For \texttt{2AC-intra} and \texttt{6pt-intra}, a technique is used to improve their numerical stability. More details can be found in the supplementary material.}. The procedure is repeated $10,000$ times. 
The numerical error for rotation is measured by $\varepsilon_{\R,\text{chordal}} = \min_i \| \R_i - \R_{\text{gt}} \|$, where $i$ counts all real solutions. 
The numerical error for translation is measure by $\varepsilon_{\mathbf{t}}$.
We did not use $\varepsilon_{\R}$ to evaluate numerical stability because the $\arccos$ function will introduce non-negligible numerical errors.
The empirical probability density functions are plotted as the function of the $\log_{10}$ estimated errors $\varepsilon_{\R,\text{chordal}}$ and $\varepsilon_{\mathbf{t}}$. 

Among the AC-based minimal solvers, the proposed \texttt{2AC-inter-56} solver has significantly better numerical stability than \texttt{2AC-inter-48} and \texttt{2AC-intra} solvers. 
Based on this result, we recommend \texttt{2AC-inter-56} as the default solver for inter-camera ACs. It will be used for the following experiments, and we refer to it as \texttt{2AC-inter}.

Among the PC-based minimal solvers, the proposed \texttt{6pt-inter-56} solver has significantly better numerical stability than \texttt{6pt-inter-48}, \texttt{6pt-intra} and \texttt{6pt-Stew{\'e}nius} solvers. 
Based on this result, we recommend \texttt{6pt-inter-56} as the default solver for inter-camera PCs. It will be used for the following experiments, and we refer to it as \texttt{6pt-inter}.
\texttt{17pt-Li} has the best numerical stability since it is a linear solver and needs the fewest calculations. The \texttt{8pt-Kneip} method based on iterative optimization is susceptible to falling into local minima and has the worst numerical stability.

\begin{figure}[tbp]
	\begin{center}
		\includegraphics[width=0.75\linewidth]{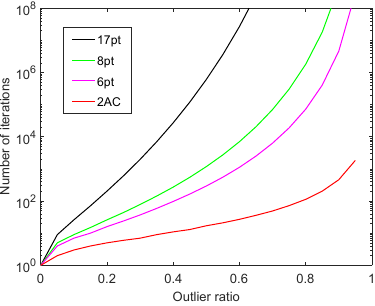}
	\end{center}
	\caption{RANSAC iteration number with respect to outlier ratio for success probability $0.99$.}
	\vspace{-0.1in}
	\label{fig:RANSACIteration}
\end{figure}  

In addition to efficiency and numerical stability, the minimal number of required geometric primitives is also an important factor for a solver. 
The iteration number $N$ of RANSAC is determined by  $N=\log(1-p)/\log(1-(1-\epsilon)^s)$, where $s$ is the minimal number of required geometric primitives, $\epsilon$ is the outlier ratio, and $p$ is the success probability. 
For a success probability $0.99$, the RANSAC iterations needed with respect to the outlier ratio are shown in Fig.\ref{fig:RANSACIteration}. 
It can be seen that the iteration number of the RANSAC estimator increases exponentially with respect to the number $s$. For example, given a percentage of outliers $\epsilon$ = $50\%$, when the solvers require $17$, $8$, $6$ and $2$ primitives, the RANSAC needs $603607$, $1177$, $292$ and $16$ iterations, respectively. 
Since the proposed AC-based solvers need only two geometric primitives, they can be used efficiently for outlier detection when integrating them into the RANSAC framework. As we will see later, the proposed AC-based solvers have better overall efficiency than PC-based solvers.

\subsection{Experiments on Synthetic Data}

We defined a simulated forward-facing  two-camera rig by following the \texttt{KITTI} autonomous driving platform~\cite{geiger2013vision}. The baseline length between the two simulated cameras is set to $1$~meter. 
The multi-camera reference frame is defined at the middle of the camera rig, and the translation between two multi-camera reference frames is $3$~meters. The resolution of the cameras is $640\times 480$ pixels, and the focal lengths are $400$~pixels. The principal points are set to the image center.

The synthetic scene is made up of a ground plane and $50$ random planes, which are randomly generated in a cube of $[-5,5] \times [-5,5] \times [10,20]$ meters, which are expressed in the respective axis of the multi-camera reference frame. We choose $50$ ACs from the ground plane and an AC from each random plane randomly. Thus, there are $100$ ACs generated randomly in the synthetic data. For each AC, a random 3D point from a plane is reprojected onto two cameras to get an image point pair. The associated affine transformation is obtained by the following procedure. 
First, four points are chosen in view~1 that are vertices of a square, where the center of the square is the projected point of an AC. 
The side length of the square is set as $30$ or $40$~pixels. A larger side length causes smaller affinity noise.
Second, the four corresponding points in view~2 are determined by the ground truth homography. Third, the sampled point pairs are contaminated by Gaussian noise. Fourth, we estimate a noisy homography using noisy point pairs. The noisy affine transformation is the first-order approximation of the noisy homography matrix.
	
A total of $1000$ trials are carried out in the synthetic experiment. In each trial, $100$ ACs are generated randomly. Two ACs for the proposed methods are selected randomly. The error is measured on the best relative pose, which produces the most inliers within the RANSAC scheme. The RANSAC scheme also allows us to select the best candidate from multiple solutions. The median of errors is used to assess the rotation and translation errors. 
In this set of experiments, the translation direction between two multi-camera references is chosen to produce either forward, sideways, or random motions. For each motion, the second view is perturbed by a random rotation. This random rotation is rotated around three axes in order, and the rotation angles range from $-10^\circ$ to $10^\circ$. 

Figure~\ref{fig:RT_sythetic} demonstrates the performance of different methods against image noise. 
Solid lines indicate using inter-camera correspondences, and dash-dotted lines indicate using intra-camera correspondences. 
We have the following observations. 
(1)~Using inter-camera correspondences has better performance than using intra-camera correspondences. 
(2)~The performance of AC-based methods is influenced by the magnitude of affinity noise, determined by the support region of sampled points. When the side length of the square is $40$~pixels, the proposed 2AC-based methods have better or comparable performance than the comparative methods. 
(3)~The proposed PC-based methods have better performance than \texttt{6pt-Stew{\'e}nius}.
(4)~The iterative optimization in \texttt{8pt-Kneip} is susceptible to falling into local minima. It performs well for forward motion. However, it does not perform well for the other two motion modes, especially when using intra-cam PCs. Sometimes the error curves are out of the display range.

\begin{figure*}[tbp]
	\begin{center}
	\includegraphics[width=0.55\linewidth]{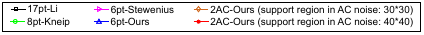}\\
	\vspace{-2pt} 
		\subfigure[${\varepsilon_{\mathbf{R}}}$]
		{
			\includegraphics[width=0.27\linewidth]{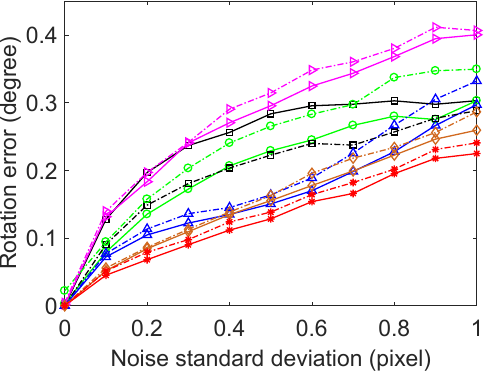}
		}
		\subfigure[${\varepsilon_{\mathbf{t}}}$]
		{
			\includegraphics[width=0.27\linewidth]{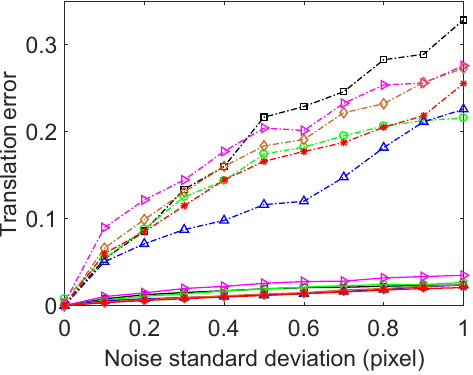}
		}
		\subfigure[$\varepsilon_{\mathbf{t},\text{dir}}$]
		{
			\includegraphics[width=0.27\linewidth]{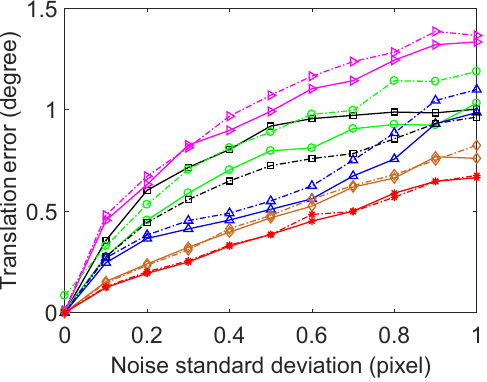}
		}
		\subfigure[${\varepsilon_{\mathbf{R}}}$]
		{
			\includegraphics[width=0.27\linewidth]{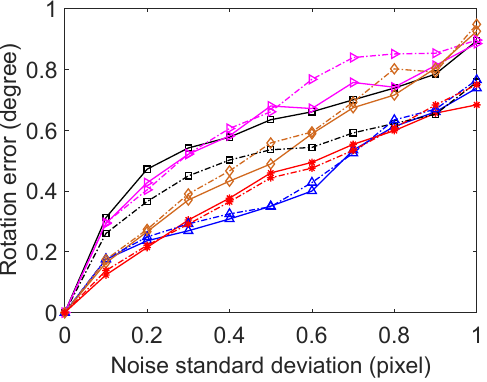}
		}
		\subfigure[${\varepsilon_{\mathbf{t}}}$]
		{
			\includegraphics[width=0.27\linewidth]{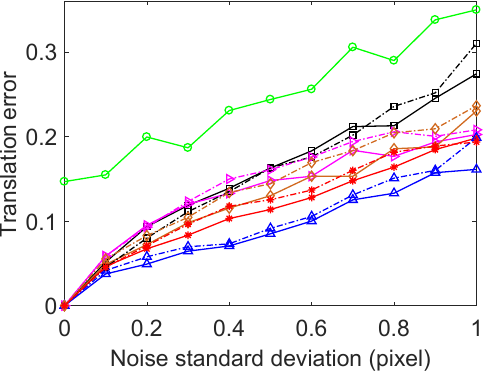}
		}
		\subfigure[$\varepsilon_{\mathbf{t},\text{dir}}$]
		{
			\includegraphics[width=0.27\linewidth]{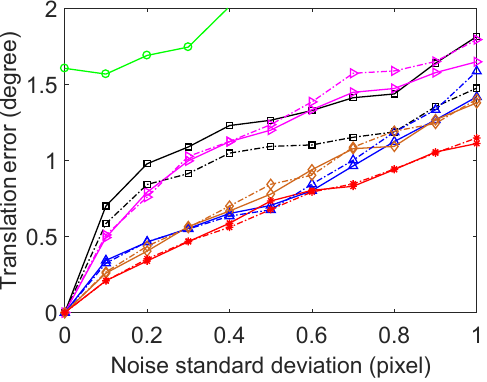}
		}
		\subfigure[${\varepsilon_{\mathbf{R}}}$]
		{
			\includegraphics[width=0.27\linewidth]{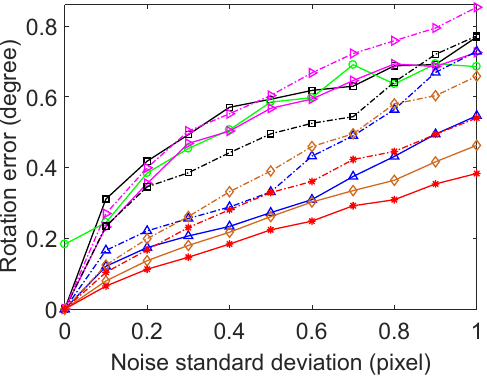}
		}
		\subfigure[${\varepsilon_{\mathbf{t}}}$]
		{
			\includegraphics[width=0.27\linewidth]{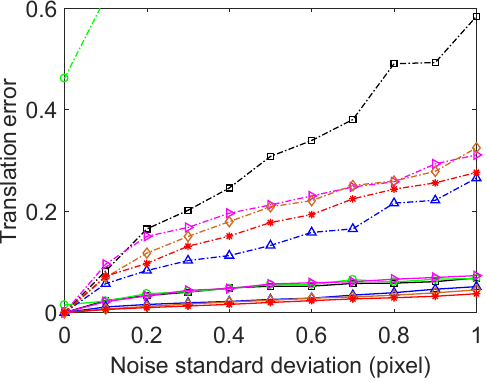}
		}
		\subfigure[$\varepsilon_{\mathbf{t},\text{dir}}$]
		{
			\includegraphics[width=0.27\linewidth]{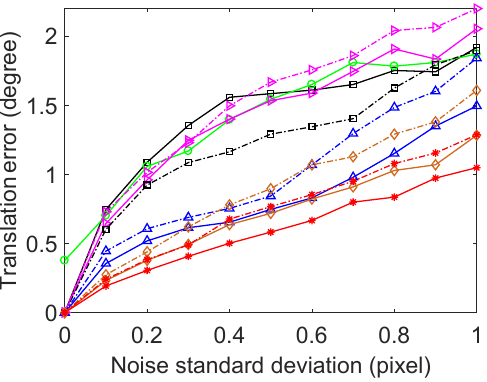}
		}
	\end{center}
	\caption{Rotation and translation error with varying image noise. The first, second, and third rows correspond to forward, sideways, and random motions, respectively. Solid lines indicate using inter-camera ACs, and dash-dotted lines indicate using intra-camera ACs.}
	\label{fig:RT_sythetic}
\end{figure*}

\subsection{Experiments on Real-World Data}

We test the performance of our methods on the \texttt{KITTI} dataset~\cite{geiger2013vision}, which consists of successive video frames from a forward-facing stereo camera. The sequences labeled from 00 to 10 that have ground truth are used for the evaluation. Therefore, the methods were tested on a total of $23000$ image pairs. The ACs between consecutive frames in each camera are established by applying the ASIFT~\cite{morel2009asift}. It can also be obtained by MSER~\cite{matas2004robust} which will be slightly less accurate but much faster to obtain~\cite{barath2016accurate}. We ignore the overlap in their fields of view and treat it as a general multi-camera system. The ACs across the two cameras are not matched, and the translation scale is not estimated as the movement between consecutive frames is small. Instead, integrating the acceleration over time from an IMU is more suitable for recovering the translation scale~\cite{nutzi2011fusion}. All the solvers have been integrated into a RANSAC scheme to deal with outliers.

\begin{table}[tbp]
	\caption{Rotation and translation error of multi-camera systems on \texttt{KITTI} sequences (unit: degree).}
	\begin{center}
	\setlength{\tabcolsep}{2.5pt}{
		\begin{tabular}{|c|cc|cc|cc|cc|cc|}
			\hline
			\multirow{2}{*}{Seq.} & \multicolumn{2}{c|}{17pt-Li} & \multicolumn{2}{c|}{8pt-Kneip} & \multicolumn{2}{c|}{6pt-Stew} & \multicolumn{2}{c|}{6pt-intra} & \multicolumn{2}{c|}{2AC-intra} \\ \cline{2-11} 
			& ${\varepsilon_{\mathbf{R}}}$ & $\varepsilon_{\mathbf{t},\text{dir}}$ & ${\varepsilon_{\mathbf{R}}}$ & $\varepsilon_{\mathbf{t},\text{dir}}$ & ${\varepsilon_{\mathbf{R}}}$ & $\varepsilon_{\mathbf{t},\text{dir}}$ & ${\varepsilon_{\mathbf{R}}}$ & $\varepsilon_{\mathbf{t},\text{dir}}$ & ${\varepsilon_{\mathbf{R}}}$ & $\varepsilon_{\mathbf{t},\text{dir}}$ \\ \hline
			00 &   0.139 & 2.412 &  0.130  &  2.400 & 0.229 & 4.007 & 0.168 & 3.311 & \textbf{0.123} & \textbf{2.291}     \\
			\rowcolor{gray!10}01 & 0.158 & 5.231 &  0.171  &  4.102& 0.762 & 41.19 & 0.335 & 16.24 & \textbf{0.139} & \textbf{2.863}     \\
			02 & 0.123 & 1.740 &  0.126 & 1.739& 0.186 & 2.508 & 0.152 & 2.294 &\textbf{0.118} & \textbf{1.658}     \\
			\rowcolor{gray!10}03 & 0.115 & 2.744 &  0.108 & 2.805& 0.265 & 6.191 & 0.158 & 4.073 &\textbf{0.104} & \textbf{2.506}     \\
			04 & 0.099 & \textbf{1.560} &  0.116 & 1.746 & 0.202 & 3.619 & 0.173 & 2.887 & \textbf{0.093} & 1.615     \\
			\rowcolor{gray!10}05 & 0.119 & 2.289 &  0.112 & 2.281 & 0.199 & 4.155 & 0.141 & 2.964 &  \textbf{0.107} &\textbf{2.216}     \\
			06 & 0.116 & 2.071 &  0.118 & 1.862 & 0.168 & 2.739 & 0.152 &  2.427 & \textbf{0.110} & \textbf{1.814}     \\
			\rowcolor{gray!10}07 & 0.119 & 3.002 &  \textbf{0.112} & 3.029 & 0.245 & 6.397 & 0.171  & 4.045 & 0.126 & \textbf{2.715}    \\
			08 & 0.116 & 2.386 &  0.111 & 2.349 & 0.196 & 3.909 & 0.151  & 3.135 & \textbf{0.091} & \textbf{2.267}           \\
			\rowcolor{gray!10}09 & 0.133 & 1.977 &  0.125 & 1.806 & 0.179 & 2.592 & 0.157 & 2.552 & \textbf{0.119} & \textbf{1.723}     \\
			10 & 0.127 & 1.889 &  \textbf{0.115} & 1.893 & 0.201 & 2.781 & 0.185   & 2.433 & 0.182 & \textbf{1.668} \\
			\hline
		\end{tabular}
	}
	\end{center}
	\label{VerticalRTErrror}
\end{table}

\begin{figure}[htbp]
	\begin{center}
		\subfigure[rotation error $\varepsilon_{\R}$]
		{
			\includegraphics[width=0.825\linewidth]{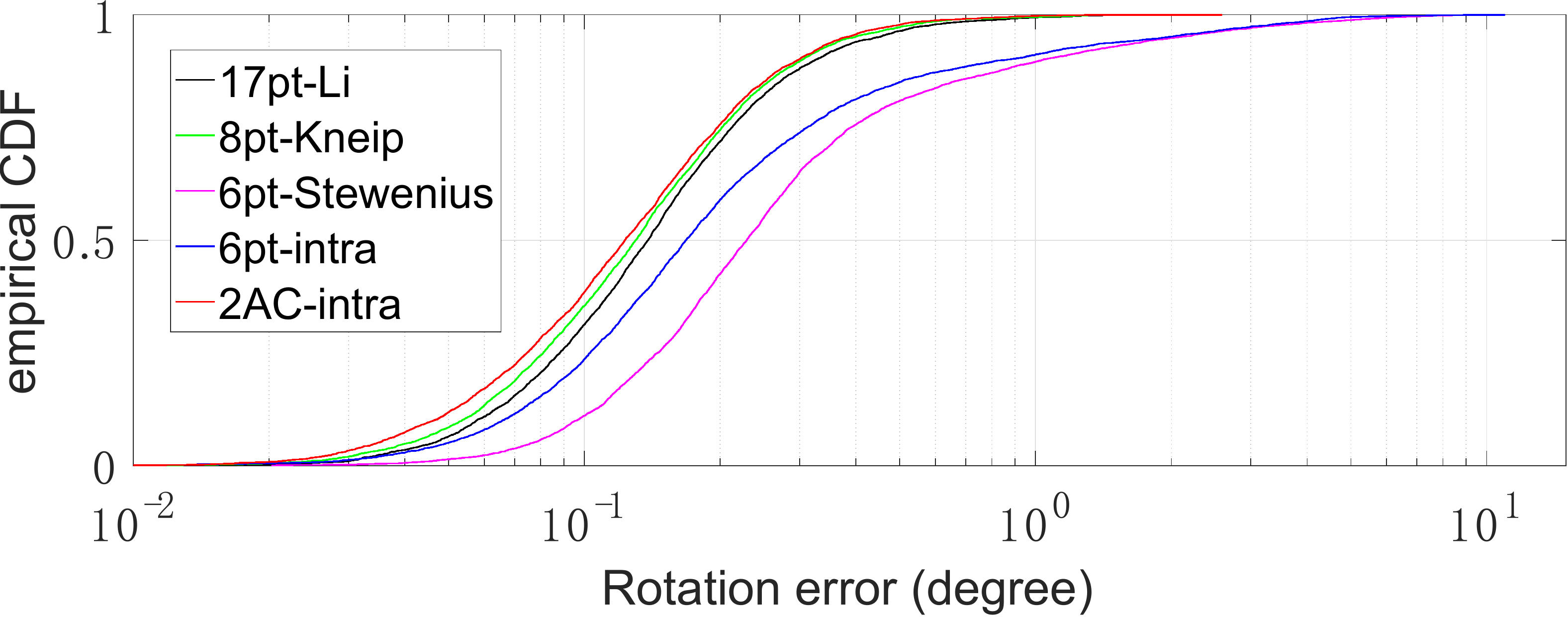}
		}
		\subfigure[translation err $\varepsilon_{\mathbf{t},\text{dir}}$]
		{
			\includegraphics[width=0.825\linewidth]{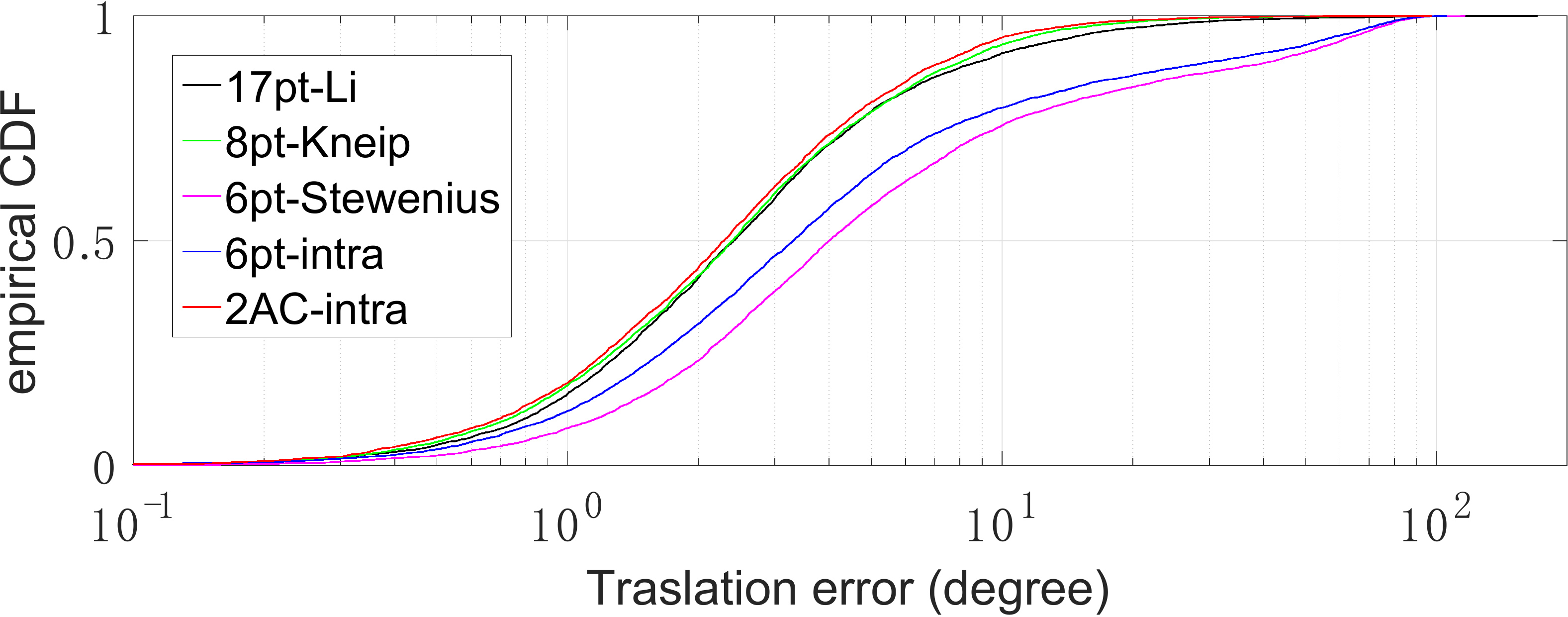}
		}
	\end{center}
	\vspace{-0.1in}
	\caption{Empirical cumulative distribution functions for KITTI sequence 00.}
	\label{fig:RTCDF}
\end{figure} 

\begin{table}[tbp]
	\caption{Runtime of RANSAC averaged over \texttt{KITTI} sequences combined with different solvers (unit: second).}
	\begin{center}
		\setlength{\tabcolsep}{3pt}{
			\begin{tabular}{lccccc}
				\toprule
				method &  17pt-Li & 8pt-Kneip &  6pt-Stew & 6pt-intra & 2AC-intra  \\
				\midrule
				mean time & 52.82 & 10.36 & 79.76 & 48.93 & {\bf 5.15} \\
				std. deviation & 2.62 & 1.59 & 4.52 & 3.17 &  {\bf 0.38} \\
				\bottomrule
			\end{tabular}
		}
	\end{center}
	\label{fig:RANSACTime}
\end{table}

The proposed \texttt{2AC-intra} and \texttt{6pt-intra} methods are compared with \texttt{17pt-Li}~\cite{li2008linear}, \texttt{8pt-Kneip}~\cite{kneip2014efficient}, and \texttt{6pt-Stew{\'e}nius}~\cite{stewenius2005solutions}. 
The results of the rotation and translation estimation are shown in Table~\ref{VerticalRTErrror}.
The \texttt{2AC-intra} offers the best overall performance among all the methods. 
The \texttt{6pt-intra} has consistently better performance than \texttt{6pt-Stew{\'e}nius}.
In Fig.~\ref{fig:RTCDF}, the empirical cumulative distribution functions for KITTI sequence~00 are shown. 
It also demonstrates the proposed \texttt{2AC-intra} offers the best overall performance in comparison to state-of-the-art methods.

The runtimes of RANSAC averaged over \texttt{KITTI} sequences combined with different solvers are shown in Table~\ref{fig:RANSACTime}.
Due to the benefits of computational efficiency, the \texttt{2AC-intra} method is suitable for finding a correct inlier set, which is then used for accurate motion estimation in visual odometry.

\section{Conclusion}
\label{sec:conclusion}

We proposed a complete solution and a series of solvers for relative pose estimation by exploiting affine correspondences. A minimum of two affine correspondences is used to estimate 6DOF relative pose of a multi-camera system. 
Two minimal solvers using point correspondences are also proposed for two-camera rigs. 
A few extensions to relative pose estimation with known rotation axis and/or unknown focal lengths are also proposed. 
All the proposed solvers are based on a unified and versatile framework.
We evaluate the proposed solvers on both synthetic data and real-world image datasets. The experimental results demonstrated that the proposed solvers for multi-camera systems provide better overall efficiency and accuracy than state-of-the-art methods.

\ifCLASSOPTIONcompsoc
 \section*{Acknowledgments}
\else
\fi

The authors thank Dr. Qi Xie at Xi'an Jiaotong University for his help regarding the proofs.

\ifCLASSOPTIONcaptionsoff
  \newpage
\fi



\bibliographystyle{IEEEtran}
%
{
	\bibliography{egbib}
}

%












\clearpage
\Large
\begin{center}
	{\bf Supplementary Material }
\end{center}
\normalsize

%


\appendices

\section{Proof}

\subsection{Proof of Theorem~1}

In this section, we use ``$\cdot$'' to represent dot product.

\begin{lemma}
Properties of cross product:

\noindent (i) $\a \cdot (\b \times \c) = 
\left|\begin{array}{ccc}
a_1 & a_2 & a_3 \\
b_1 & b_2 & b_3 \\
c_1 & c_2 & c_3
\end{array}\right|
= 
\left|\begin{array}{ccc}
a_1 & b_1 & c_1 \\
a_2 & b_2 & c_2 \\
a_3 & b_3 & c_3
\end{array}\right|$, where $\a = [a_1, a_2, a_3]^T$, $\b = [b_1, b_2, b_3]^T$, $\c = [c_1, c_2, c_3]^T$.

\noindent(ii) $\a \times (\b \times \c) = \b (\a \cdot \c) - \c (\a \cdot \b)$. This is known as triple product expansion, or Lagrange's formula.

\noindent(iii) $\a \cdot (\b \times \c) = \b \cdot (\c \times \a) = \c \cdot (\a \times \b)$. The scalar triple product is unchanged under a circular shift of its three operands $(\a, \b, \c)$.

\noindent(iv) $\a \times (\b + \c) = (\a \times \b) + (\a\times \c)$.

\noindent(v) $\a \times \b = [\a]_{\times} \b = -[\b]_{\times} \a$.
\label{lemma:crossproduct}
\end{lemma}

The above properties can be found in [\url{https://en.wikipedia.org/wiki/Cross_product}] and [\url{https://en.wikipedia.org/wiki/Triple_product}].

\begin{lemma}
\label{lemma:ortho_cross_prod}
Suppose $\Q$ is a matrix satisfying
\begin{align}
\Q \Q^T = s^2 \mathbf{I} = 
\begin{bmatrix}
s^2 & 0 & 0 \\
0 & s^2 & 0 \\
0 & 0 & s^2 \\
\end{bmatrix},
\ s > 0
\label{eq:ortho}
\end{align}
then
\begin{align}
\Q_i \times \Q_j = 
\begin{cases}
s I_{k} \Q_k & \text{if } i\neq j,\\
0 & \text{if } i = j.\\
\end{cases}
\label{eq:cross_product_ortho_matrix}
\end{align}
where $\mathbf{I}$ is an identity matrix, $\Q_i$ is the $i$-th row (or column) of $\Q$, $I_{k} = +1$ or $-1$, and $k = \{1,2,3\} \setminus \{i, j\}$.
\end{lemma}
\begin{proof}
When $i \neq j$, the direction of $\Q_k$ is perpendicular to both $\Q_i$ and $\Q_j$ according to Eq.~\eqref{eq:ortho}. Thus the direction of $\Q_i \times \Q_j$ is $ \mathbf{n} := \frac{I_{k} \Q_k}{\|\Q_k\|} = \frac{1}{s} I_{k} \Q_k$. According to the definition of cross product, $\Q_i \times \Q_j = \sin\langle \Q_i, \Q_j \rangle \cdot \|\Q_i\| \cdot \|\Q_j\| \cdot \mathbf{n} = \sin(\pi/2) \cdot s \cdot s \cdot \frac{1}{s} I_{k} \Q_k = s I_k \Q_k$.
When $i = j$, we have $\Q_i \times \Q_j = \Q_i \times \Q_i = 0$. 
\end{proof}

\begin{lemma}
Suppose matrix $\Q_{3\times 3}$ satisfies
\begin{align}
\Q \Q^T = s^2 \mathbf{I}, \ s > 0
\end{align}
where $s$ is a polynomial of some variables.
Suppose $\a_1$, $\a_2$, $\a_3$, $\b_1$, $\b_2$, and $\b_3 \in \mathbb{R}^3$ are arbitrary non-zero vectors. Then the determinant of 
\begin{align}
\mathbf{N} = 
\begin{bmatrix}
(\a_1 \times \Q \b_1)^T \\ 
(\a_2 \times \Q \b_2)^T \\ 
(\a_3 \times \Q \b_3)^T
\end{bmatrix}
\end{align}
has a factor $s$. 
\label{lemma:det_pt}
\end{lemma}
\begin{proof}

According to the properties of cross product in Lemma~\ref{lemma:crossproduct}, we have
\begin{align}
& \det(\mathbf{N}) = \det(\mathbf{N}^T) \nonumber \\
= & \det([
\a_1 \times \Q \b_1, \a_2 \times \Q \b_2, 
\a_3 \times \Q \b_3
]) \nonumber \\
= & (\a_1 \times \Q \b_1) \cdot [(\a_2 \times \Q \b_2) \times (\a_3 \times \Q \b_3)] \nonumber \\
= & (\a_1 \times \Q \b_1) \cdot  \nonumber \\
& \{ \a_3 [(\a_2 \times \Q \b_2) \cdot \Q \b_3] - \Q \b_3 [(\a_2 \times \Q \b_2) \cdot \a_3] \} \nonumber \\
= & [(\a_1 \times \Q \b_1) \cdot \a_3] [(\a_2 \times \Q \b_2) \cdot \Q \b_3] \nonumber \\
& - [(\a_1 \times \Q \b_1) \cdot \Q \b_3] [(\a_2 \times \Q \b_2) \cdot \a_3] \nonumber \\
= & [(\a_1 \times \Q \b_1) \cdot \a_3] [\a_2 \cdot (\Q \b_2 \times \Q \b_3)] \nonumber \\
& - [\a_1 \cdot (\Q \b_1 \times \Q \b_3] [(\a_2 \times \Q \b_2) \cdot \a_3].
\end{align}

From the above equation, if all entires of $\Q \b_2 \times \Q \b_3$ and $\Q \b_1 \times \Q \b_3$ have a common factor $s$, the determinant $\det(\mathbf{N})$ should has the factor $s$ too.
In the following, we prove that all entries of $\Q \a \times \Q \b$ have the factor $s$ for arbitrary vectors $\a = [a_1, a_2, a_3]^T$ and $\b = [b_1, b_2, b_3]^T$.
Since $\Q$ satisfies Eq.~\eqref{eq:ortho} in Lemma~\ref{lemma:ortho_cross_prod}, we have 
\begin{align}
& \Q \a \times \Q \b \nonumber \\
= & \left( \sum_{i=1}^3 a_i \Q_i \right) \times \left( \sum_{i=1}^3 b_i \Q_i \right) \nonumber \\
= & \sum_{i=1}^3 \sum_{j\neq i} a_i b_j \Q_i \times \Q_j \nonumber \\
= & \sum_{i=1}^3 \sum_{j\neq i, k \notin \{i,j\}} a_i b_j s I_k \Q_k \nonumber \\
= & s[ I_1(a_2 b_3 - a_3 b_2) \Q_1 + I_2(a_3 b_1 - a_1 b_3) \Q_2 \nonumber \\
& \quad + I_3(a_1 b_2 - a_2 b_1) \Q_3],
\end{align}
where $I_k$ equals $+1$ or $-1$. 
It can be seen that $s$ is a common factor for all entries in $\Q \a \times \Q \b \nonumber$.
\end{proof}

\begin{mytheorem}{1}[]
	Suppose a matrix $\Q_{3\times 3}$ satisfies
	\begin{align}
	\Q \Q^T = s^2 \mathbf{I}, \ s > 0
	\end{align}
	where $s$ is a polynomial of some variables.
	Suppose $\a_1^{(i)}$, $\a_2^{(i)}$, $\a_3^{(i)}$, $\b_1^{(i)}$, $\b_2^{(i)}$, and $\b_3^{(i)} \in \mathbb{R}^3$ are arbitrary non-zero vectors. 
	Then the polynomial of determinant
	\begin{align}
	\mathbf{N} = \begin{bmatrix}
	\left(\sum_{i=1}^{m_1} \a_1^{(i)} \times \Q \b_1^{(i)}\right)^T \\ \left(\sum_{i=1}^{m_2} \a_2^{(i)} \times \Q \b_2^{(i)}\right)^T \\ \left(\sum_{i=1}^{m_3} \a_3^{(i)} \times \Q \b_3^{(i)}\right)^T
	\end{bmatrix}
	\label{eq:ggform_app}
	\end{align}
	has a factor $s$, where $m_1, m_2, m_3$ are positive integers.
	\label{lemma:general_form}
\end{mytheorem}

\begin{proof}
	We prove this proposition for the case of $m_1 = m_2 = m_3 = 2$. The proof can be extended to other cases straightforwardly.
	Rewrite $\mathbf{N}$ as the following form 
	\begin{align}
	\mathbf{N} =
	\begin{bmatrix}
	(\a_1 \times \Q \b_1 + \c_1 \times \Q \d_1)^T \\
	(\a_2 \times \Q \b_2 + \c_2 \times \Q \d_2)^T \\
	(\a_3 \times \Q \b_3 + \c_3 \times \Q \d_3)^T
	\end{bmatrix}.
	\label{equ:general_form_basic}
	\end{align}
	Its determinant is 
	\begin{align}
	& \det(\mathbf{N}) = \det(\mathbf{N}^T) \nonumber \\
	= & (\a_1 \times \Q \b_1 + \c_1 \times \Q \d_1) \cdot \nonumber \\
	& \left[ (\a_2 \times \Q \b_2 + \c_2 \times \Q \d_2) \times (\a_3 \times \Q \b_3 + \c_3 \times \Q \d_3) \right] \nonumber \\
	= & (\a_1 \times \Q \b_1 + \c_1 \times \Q \d_1) \cdot \nonumber \\
	& [(\a_2 \times \Q \b_2) \times (\a_3 \times \Q \b_3) + (\a_2 \times \Q \b_2) \times (\c_3 \times \Q \d_3) \nonumber \\
	+ & \ (\c_2 \times \Q \d_2) \times (\a_3 \times \Q \b_3) + (\c_2 \times \Q \d_2) \times (\c_3 \times \Q \d_3)] \nonumber \\
	= & \det([\a_1 \times \Q \b_1, \a_2 \times \Q \b_2, \a_3 \times \Q \b_3]) \nonumber \\
	& + \det([\a_1 \times \Q \b_1, \a_2 \times \Q \b_2, \c_3 \times \Q \d_3]) \nonumber \\
	& + \det([\a_1 \times \Q \b_1, \c_2 \times \Q \d_2, \a_3 \times \Q \b_3]) \nonumber \\
	& + \det([\a_1 \times \Q \b_1, \c_2 \times \Q \d_2, \c_3 \times \Q \d_3]) \nonumber \\
	& + \det([\c_1 \times \Q \d_1, \a_2 \times \Q \b_2, \a_3 \times \Q \b_3]) \nonumber \\
	& + \det([\c_1 \times \Q \d_1, \a_2 \times \Q \b_2, \c_3 \times \Q \d_3]) \nonumber \\
	& + \det([\c_1 \times \Q \d_1, \c_2 \times \Q \d_2, \a_3 \times \Q \b_3]) \nonumber \\
	& + \det([\c_1 \times \Q \d_1, \c_2 \times \Q \d_2, \c_3 \times \Q \d_3]).
	\label{eq:det_decomp}
	\end{align}
	According Lemma~\ref{lemma:det_pt}, each term in the right-hand side of Eq.~\eqref{eq:det_decomp} has a factor of $s$. Thus $\det(\mathbf{N})$ also has this factor.
\end{proof}

\subsection{Proof of Theorem~2}

\begin{lemma}
Given an affine correspondence $(\x, \x', \A)$ and essential matrix $\E = [\mathbf{t}]_\times \R$, the three constraints
\begin{subequations}
	\begin{empheq}[left=\empheqlbrace]{align}
	& \x'^T \E \x = 0 \label{equ:app_sub1} \\
	& (\E^T \x')_{(1:2)} + \A^{T} (\E \x)_{(1:2)} = \mathbf{0} \label{equ:app_sub2}
	\end{empheq}
\end{subequations}
can be reformulated as
\begin{align}
-\begin{bmatrix}
(\x' \times \R \x)^T \\
(\x' \times \R \c_1 + \a_1 \times \R \x)^T \\
(\x' \times \R \c_2 + \a_2 \times \R \x)^T
\end{bmatrix} \mathbf{t} = \mathbf{0},
\label{eq:ptconst}
\end{align}
where $\c_1 = [1, 0, 0]^T$, $\c_2 = [0, 1, 0]^T$, $\a_1 = [\A_{11}, \A_{21}, 0]^T$, and $\a_2 = [\A_{12}, \A_{22}, 0]^T$.
\label{lemma:det_ac}
\end{lemma}

\begin{proof}

For Eq.~\eqref{equ:app_sub1}, we have
\begin{align}
& \x'^T \E \x \nonumber \\
= & \x'^T [\mathbf{t}]_\times \R \x 
= (\R \x)^T (\x'^T [\mathbf{t}]_\times)^T \nonumber \\
= & -\x^T \R^T [\mathbf{t}]_\times \x' 
=  \x^T \R^T [\x']_\times \mathbf{t} \nonumber \\
= & (\R \x \times \x')^T \mathbf{t} 
=  -(\x' \times \R \x)^T \mathbf{t}.
\label{eq:pt_term}
\end{align}
For the first term of Eq.~\eqref{equ:app_sub2}, we have
\begin{align}
& \E^T \x' \nonumber \\
= & ([\mathbf{t}]_\times \R)^T \x'
= -\R^T [\mathbf{t}]_\times \x' \nonumber \\
= & \mathbf{I} \R^T [\x']_\times \mathbf{t},
\end{align}
where $\mathbf{I}$ is an identity matrix.
In the above equation, the $i$-th entry of the right-hand side can be reformulated as
\begin{align}
\c_i^T \R^T [\x']_\times \mathbf{t} = -(\x'\times \R \c_i)^T \mathbf{t}.
\label{eq:af_term1}
\end{align}
For the second term of Eq.~\eqref{equ:app_sub2}, its $i$-th ($i = 1,2$) entry is $\a_i^T \E \x$. Similar to the derivation in Eq.~\eqref{eq:pt_term}, we have
\begin{align}
\a_i^T \E \x = -(\a_i \times \R \x)^T \mathbf{t}.
\label{eq:af_term2}
\end{align}
By combining Eq.~\eqref{eq:pt_term},~\eqref{eq:af_term1}, and~\eqref{eq:af_term2}, the proof is completed.
\end{proof}

\begin{mytheorem}{2}[]
	Suppose $\mathbf{N}$ is an arbitrary $3\times 3$ submatrix of $\overline{\M}$, the polynomial $\det(\mathbf{N})$ has a factor $q_x^2 + q_y^2 + q_z^2 + 1$.
\end{mytheorem}
\begin{proof}
	Denote 
	\begin{align}
		& \quad \Q = (q_x^2 + q_y^2 + q_z^2 + 1) \R \nonumber \\
		= & \begin{bmatrix}{1+q_x^2-q_y^2-q_z^2}&{2{q_x}{q_y}-2{q_z}}&{2{q_y}+2{q_x}{q_z}}\\
		{2{q_x}{q_y}+2{q_z}}&{1-q_x^2+q_y^2-q_z^2}&{2{q_y}{q_z}-2{q_x}}\\
		{2{q_x}{q_z}-2{q_y}}&{2{q_x}+2{q_y}{q_z}}&{1-q_x^2-q_y^2+q_z^2}
		\end{bmatrix},
	\end{align}
	It can be verified that $\Q \Q^T = s^2 \mathbf{I}$ with $s = q_x^2 + q_y^2 + q_z^2 + 1$.
	
	Recall that $\overline{\M}$ is composed of constraints from two affine correspondences, and $\mathbf{N}$ is a $3\times 3$ submatrix of $\overline{\M}$. According to Lemma~\ref{lemma:det_ac}, 
	$\mathbf{N}$ has the form of 
	\begin{align}
	\mathbf{N} =
	\begin{bmatrix}
	(\a_1 \times \Q \b_1 + \c_1 \times \Q \d_1)^T \\
	(\a_2 \times \Q \b_2 + \c_2 \times \Q \d_2)^T \\
	(\a_3 \times \Q \b_3 + \c_3 \times \Q \d_3)^T
	\end{bmatrix}.
	\label{equ:general_form}
	\end{align}
	Note that $\Q$ instead of $\R$ appears in this matrix because the scale factor $\frac{1}{q_x^2+q_y^2+q_z^2+1}$ is ignored during the equation system construction. 
	Since all the conditions in Lemma~\ref{lemma:general_form} are satisfied, $\det(\mathbf{N})$ has a factor $s$.
\end{proof}

\subsection{Proof of Theorem~4}

\begin{lemma}
	For a multi camera system, denote the constraints introduced by the $k$-th affine correspondence as
	\begin{subequations}
		\begin{empheq}[left=\empheqlbrace]{align}
		& \x_k'^T \E_k \x_k = 0 \label{equ:supp_sub1} \\
		& (\E_k^T \x_k')_{(1:2)} + \A_k^{T} (\E_k \x_k)_{(1:2)} = \mathbf{0} \label{equ:supp_sub2}
		\end{empheq}
	\end{subequations}
	where 
	\begin{align}
	\E_k  = \Q_{i'}^T (\R [\mathbf{s}_i]_\times + [\mathbf{t} - \mathbf{s}_{i'}]_\times \R) \Q_i.
	\label{eq:supp_essential_matrix2}
	\end{align}
	Denote
	\begin{align}
	\tilde{\A}_k =
	\begin{bmatrix}
	\A_k & \mathbf{0}_{2\times 1} \\
	\mathbf{0}_{1\times 2} & 1
	\end{bmatrix}.
	\end{align}
	Then these constraints can be reformulated as
	\begin{align}
	-\begin{bmatrix*}[l]
	(\Q_{i'} \x'_k \times \R \Q_i \x_k)^T & o_1(\R) \\
	(\Q_{i'}\x'_k \times \R \c_1 + \a_1 \times \R \Q_i \x_k)^T & o_2(\R) \\
	(\Q_{i'} \x'_k \times \R \c_2 + \a_2 \times \R \Q_i \x_k)^T & o_3(\R)
	\end{bmatrix*} 
	\begin{bmatrix}
	\mathbf{t} \\ 1
	\end{bmatrix} = \mathbf{0},
	\label{eq:gcam_ac_reform2}
	\end{align}
	where $\c_i$ is the $i$-th column of $\Q_i$, $\a_i$ is the $i$-th row of $\tilde{\A}_k^T \Q_{i'}$, and $o_i(\R)$ includes all terms depending on $\R$.
	\label{lemma:gcam_constraint}
\end{lemma}
\begin{proof}
	Equation~\eqref{equ:supp_sub1} can be reformulated as
	\begin{align}
	& \x_k'^T \E_k \x_k \nonumber \\
	= & \x_k'^T (\Q_{i'}^T [\mathbf{t}]_\times \R \Q_i) \x_k + o_1(\R) \nonumber \\
	= & (\Q_{i'} \x_k')^T [\mathbf{t}]_\times \R (\Q_i \x_k) + o_1(\R).
	\end{align}
	The expression in Eq.~\eqref{equ:supp_sub2} can be reformulated as
	\begin{align}
	& (\E_k^T \x_k') + \tilde{\A}_k^{T} (\E_k \x_k) \nonumber \\
	= & (\Q_{i'}^T [\mathbf{t}]_\times \R \Q_i)^T \x_k' + \tilde{\A}_k^{T} (\Q_{i'}^T [\mathbf{t}]_\times \R \Q_i) \x_k + \mathbf{o}_2(\R) \nonumber \\
	= & \Q_i^T ([\mathbf{t}]_\times \R )^T \Q_{i'} \x_k' + \tilde{\A}_k^{T} \Q_{i'} ([\mathbf{t}]_\times \R ) \Q_i \x_k + \mathbf{o}_2(\R) 
	\end{align}
	where $\mathbf{o}_2(\R)$ includes all terms depending on $\R$.
	
	Similar to the derivation in Lemma~\ref{lemma:det_ac}, Eqs.~\eqref{equ:supp_sub1}\eqref{equ:supp_sub2} can be reformulated as
	\begin{align}
	-\begin{bmatrix*}[l]
	(\Q_{i'} \x'_k \times \R \Q_i \x_k)^T & o_1(\R) \\
	(\Q_{i'}\x'_k \times \R \c_1 + \a_1 \times \R \Q_i \x_k)^T & [\mathbf{o}_2(\R)]_1 \\
	(\Q_{i'} \x'_k \times \R \c_2 + \a_2 \times \R \Q_i \x_k)^T & [\mathbf{o}_2(\R)]_2
	\end{bmatrix*} 
	\begin{bmatrix}
	\mathbf{t} \\ 1
	\end{bmatrix} = \mathbf{0},
	\label{eq:gcam_ac_reform}
	\end{align}
	where $\c_i$ is the $i$-th column of $\Q_i$, and $\a_i$ is the $i$-th row of $\tilde{\A}_k^T \Q_{i'}$.
\end{proof}

\begin{mytheorem}{4}[]
	Suppose $\mathbf{N}$ is an arbitrary $4\times 4$ submatrix of $\M$ or an arbitrary $3\times 3$ submatrix of the first three columns of $\M$, the polynomial $\det(\mathbf{N})$ has a factor $q_x^2 + q_y^2 + q_z^2 + 1$.
\end{mytheorem}
\begin{proof}
	Denote 
	\begin{align}
	& \quad \Q = (q_x^2 + q_y^2 + q_z^2 + 1) \R \nonumber \\
	= & \begin{bmatrix}{1+q_x^2-q_y^2-q_z^2}&{2{q_x}{q_y}-2{q_z}}&{2{q_y}+2{q_x}{q_z}}\\
	{2{q_x}{q_y}+2{q_z}}&{1-q_x^2+q_y^2-q_z^2}&{2{q_y}{q_z}-2{q_x}}\\
	{2{q_x}{q_z}-2{q_y}}&{2{q_x}+2{q_y}{q_z}}&{1-q_x^2-q_y^2+q_z^2}
	\end{bmatrix},
	\end{align}
	It can be verified that $\Q \Q^T = s^2 \mathbf{I}$ with $s = q_x^2 + q_y^2 + q_z^2 + 1$.
	
	Recall that $\M$ is composed of constraints from two affine correspondences. According to Lemma~\ref{lemma:gcam_constraint},  the constraints has the form of Eq.~\eqref{eq:gcam_ac_reform2}.
	So $\mathbf{M}$ has a form of 
	\begin{align}
	\mathbf{M} =
	\begin{bmatrix*}[l]
	(\a_1 \times \Q \b_1 + \c_1 \times \Q \d_1)^T & o_1(\Q) \\
	(\a_2 \times \Q \b_2 + \c_2 \times \Q \d_2)^T & o_2(\Q) \\
	(\a_3 \times \Q \b_3 + \c_3 \times \Q \d_3)^T & o_3(\Q) \\
	(\a_4 \times \Q \b_4 + \c_4 \times \Q \d_4)^T & o_4(\Q) \\
	(\a_5 \times \Q \b_5 + \c_5 \times \Q \d_5)^T & o_5(\Q) \\
	(\a_6 \times \Q \b_6 + \c_6 \times \Q \d_6)^T & o_6(\Q)
	\end{bmatrix*}.
	\label{equ:general_form2}
	\end{align}
	Note that $\Q$ instead of $\R$ appears in this matrix. The reason is that the scale factor $\frac{1}{q_x^2+q_y^2+q_z^2+1}$ is ignored during the construction of the equation system. 
	In Eq.~\eqref{equ:general_form2}, the first three columns of coefficient matrix has the form of Eq.~\eqref{eq:ggform_app}. Lemma~\ref{lemma:general_form} can be used directly. Thus we prove that the proposition is correct when $\mathbf{N}$ is an arbitrary $3\times 3$ submatrix of the first three columns of $\M$.
	
	Next, we prove that the proposition is correct when $\mathbf{N}$ is an arbitrary $4\times 4$ submatrix of $\M$. The determinant can be calculated by the Laplace expansion along the fourth column, and it is the summation of each entry in the fourth column multiplied with the corresponding cofactor. Note that we have proved there is a factor $q_x^2+q_y^2+q_z^2+1$ for the determinant of an arbitrary $3\times 3$ submatrix from the first three columns of $\M$. Then any cofactor has the factor $q_x^2+q_y^2+q_z^2+1$.
\end{proof}

\subsection{Proof of Theorem~5}

\begin{mytheorem}{5}[]
	For non-degenerate cases, $\rank(\mathbf{N}) = 2$, $\forall \mathbf{N} \in \mathcal{S}$. By Matlab syntax, $\mathcal{S}$ is a set whose elements satisfying $\mathbf{N} = \widehat{\M}([k_1,k_2,k_3], 1:3)$, where $k_1$-th, $k_2$-th, and $k_3$-th PCs are captured by the same perspective cameras across two views and $k_1 < k_2 < k_3$. 
\end{mytheorem}
\begin{proof}
	Let us investigate an arbitrary element in $\mathcal{S}$. Denote it is the $k$-th element $\mathbf{N}_k$ in $\mathcal{S}$. The extrinsic parameter of the related perspective camera in view~1 is $[\Q_i | \mathbf{s}_i]$, and the extrinsic parameter of the related perspective camera in view~2 is $[\Q_{i'} | \mathbf{s}_{i'}]$. 
	
	First we prove that $\rank(\mathbf{N}_k) \le 2$. To achieve this goal, we need to prove that the null space of $\mathbf{N}_k$ is not empty.
	Since $k_1$-th, $k_2$-th, and $k_3$-th PCs are captured by the same perspective cameras across two views, their related essential matrices are the same.
	According to Eq.~\eqref{eq:essential_ek_pc},
	 the essential matrix is
	\begin{align}
	\E_k = \Q_{i'}^T [\mathbf{t} + \R \mathbf{s}_i - \mathbf{s}_{i'}]_\times \R \Q_i.
	\end{align}
	Denote 
	\begin{align}
	\bar{\mathbf{t}} \triangleq \mathbf{t} + \R \mathbf{s}_i - \mathbf{s}_{i'},
	\end{align}
	then we have
	\begin{align}
	\E_k = \Q_{i'}^T [\bar{\mathbf{t}}]_\times \R \Q_i.
	\label{eq:new_essential_matrix_app}
	\end{align}
	Substituting Eq.~\eqref{eq:new_essential_matrix_app} into Eq.~\eqref{eq:essential_mat_pc},
	we obtain three equations for the 3PCs. Each monomial in these three equations is linear with one entry of vector $\bar{\mathbf{t}}$, and there is no constant term. Thus these equations can be formulated as
	\begin{align}
	&\frac{1}{q_x^2+q_y^2+q_z^2+1} \mathbf{A}_{k} \bar{\mathbf{t}} = \mathbf{0},
	\label{eq:null_space_app} 
	\\
	\Rightarrow & \mathbf{A}_k (\mathbf{t} + \R \mathbf{s}_i - \mathbf{s}_{i'}) = \mathbf{0}, \\
	\Rightarrow & 
	\begin{bmatrix}
	\mathbf{A}_k & \mathbf{A}_k(\R \mathbf{s}_i - \mathbf{s}_{i'})
	\end{bmatrix}
	\begin{bmatrix}
	\mathbf{t} \\ 1
	\end{bmatrix}
	= \mathbf{0}.
	\label{eq:equ_qxqyqz1_new_form_app}
	\end{align}
	By comparing the construction procedure of Eq.~\eqref{eq:equ_qxqyqz1_pc}
	 and Eq.~\eqref{eq:equ_qxqyqz1_new_form_app}, we can see that
	\begin{align}
	\mathbf{A}_k = \widehat{\M}([k_1,k_2,k_3], 1:3) = \mathbf{N}_k.
	\end{align}
	Substituting this equation into Eq.~\eqref{eq:null_space_app}, we can see that the null space of $\mathbf{N}_k$ is not empty. 
	
	Next we prove that $\rank(\mathbf{N}_k) \ge 2$. We achieve this goal using proof by contradiction. If $\rank(\mathbf{N}_k) \le 1$, then $\rank(\widehat{\M}([k_1,k_2,k_3], 1:4)) \le 2$ considering that $\widehat{\M}([k_1,k_2,k_3], 1:4)$ has one more column than $\mathbf{N}_k$. This means the three ACs provides at most two independent constraints for the relative pose. This cannot be true for non-degenerate cases, so the assumption that $\rank(\mathbf{N}_k) \le 1$ is wrong.
\end{proof}

\section{Numerical Stability Improvement for \texttt{2AC+intra} and \texttt{6pt+intra}}

The numerical stability of \texttt{2AC+intra} and \texttt{6pt+intra} are more complicated than other cases. The \texttt{GrevLex} monomial orderings is exploited during solver generation which reads as follows: which is $q_x > q_y > q_z$. Under such an ordering, we found that when one of the following two conditions is satisfied, their solvers are numerically unstable.
Condition~1: the $z$-coordinates of $\mathbf{s}_1$ and $\mathbf{s}_2$ are the same. Condition~2: the $x$- and $y$-coordinates of $\mathbf{s}_1$ and $\mathbf{s}_2$ are the same, i.e., $(\mathbf{s}_1)_{(1:2)} = (\mathbf{s}_2)_{(1:2)}$. 

We developed a technique to improve the numerical stability of these two solvers. Our idea is to destroy the two conditions mentioned above by defining a new reference for the multi-camera system. The new reference is uniquely determined by satisfying the following two conditions. (i) Its origin is the middle point of two cameras in the two-camera rig. (ii) Two focal points of perspective cameras lie on points $\frac{L}{2\sqrt{3}} [-1, -1, -1]^T$ and $\frac{L}{2\sqrt{3}} [1, 1, 1]^T$ in the new reference, respectively. Here $L = \|\mathbf{s}_1 - \mathbf{s}_2 \|$ is the distance between two focal points of the two-camera rig. 
In this new reference, the two conditions causing numerical instability are destroyed as largely as possible. 
After the relative pose in the new reference is obtained, we convert the relative pose to the original reference. 
This technique is applied to \texttt{2AC+intra} and \texttt{6pt+intra} only since it slightly decreases the numerical stability for other cases and increases calculation. 

\subsection{Reference Transformation}
We introduce more details about reference transformation.
The coordinates of focal points in the original frame are $\mathbf{s}_1$ and $\mathbf{s}_2$. We aim to define a new reference, in which these two points has coordinates $\frac{L}{2\sqrt{3}} [-1, -1, -1]^T$ and $\frac{L}{2\sqrt{3}} [1, 1, 1]^T$, respectively. Here $L = \|\mathbf{s}_1 - \mathbf{s}_2 \|$ is the distance between two focal points of the two-camera rig.

Denote the unit direction vector connecting two focal points in the original frame and new frame as $\a = \frac{1}{L} (\mathbf{s_2} - \mathbf{s}_1)$, $\b = \frac{1}{\sqrt{3}}[1, 1, 1]^T$, respectively. The rotation can be found by axis-angle parameterization. Denote
\begin{align}
& \mathbf{v} = \frac{\a \times \b}{\| \a \times \b \|}, \\
& s = \| \a \times \b \|, \\
& c = \a \cdot \b.
\end{align}
Here $\mathbf{v}$ is the rotation axis. $s$ and $c$ are sine and cosine of the rotation angle, respectively.
According to Rodrigues' rotation formula,
the transformation rotation is
\begin{align}
\R_0 = \mathbf{I} + s [\mathbf{v}]_\times + (1-c) [\mathbf{v}]_\times^2.
\end{align}
For the middle point of the line segment between two focal points, its coordinates are $\frac{1}{2}(\mathbf{s}_1 + \mathbf{s}_2)$ and $[0,0,0]^T$ in the original frame and new frame, respectively.
Thus the translation is 
\begin{align}
\mathbf{t}_0 = -\frac{1}{2} \R_0 (\mathbf{s}_1 + \mathbf{s}_2).
\end{align}
The Matlab code is listed below.

{\scriptsize
\begin{lstlisting}
function [R, t] = find_optimal_transformation3(s1, s2)
% Input
%  s1 and s2: focal points of the perspective cameras
%  in the original frame
% Output
%  R and t satisfy: the direction of R*s1+t and R*s2+t
%  are [-1;-1;-1] and [1;1;1], respectively.
m = (s1+s2)/2;
a = s2-s1;
a = a/norm(a);
b = [1;1;1]/sqrt(3);
v = cross(a,b);
s = norm(v);
if (s>1e-10)
    v = v/s;
    SV = [0 -v(3) v(2); v(3) 0 -v(1); -v(2) v(1) 0];
    R = eye(3) + SV*s + SV^2*(1-dot(a,b));
else
    R = eye(3);
end
t = -R*m;
\end{lstlisting}
}

\subsection{Experiments}
\begin{figure}[htbp]
	\begin{center}
		\subfigure[rotation error $\varepsilon_{\mathbf{R}, \text{chordal}}$]
		{
			\includegraphics[width=0.47\linewidth]{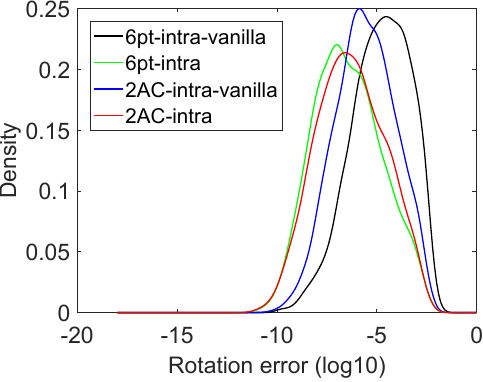}
		}
		\subfigure[translation error $\varepsilon_{\mathbf{t}}$]
		{
			\includegraphics[width=0.47\linewidth]{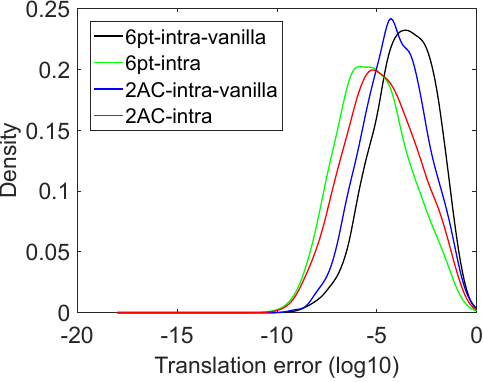}
		}
	\end{center}
	\vspace{-0.1in}
	\caption{Probability density functions over pose estimation errors on noise-free observations using reference transformation. The horizontal axis represents the $\log_{10}$ errors and the vertical axis represents the density.}
	\label{fig:Numerical3}
\end{figure}

Fig.~\ref{fig:Numerical3} reports the numerical stability of the solvers on noise-free observations. Vanilla version means using the solvers without reference transformation. 
The proposed reference transformation can reduce the numerical errors significantly. 
For \texttt{6pt+intra}, the mode of $\varepsilon_{\mathbf{R}, \text{chordal}}$ is decreased from $10^{-4.5}$ to $10^{-7}$, and the mode of $\varepsilon_{\mathbf{t}}$ is decreased from $10^{-3.5}$ to $10^{-5.9}$.
For \texttt{2AC+intra}, the mode of $\varepsilon_{\mathbf{R}, \text{chordal}}$ is decreased from $10^{-5.9}$ to $10^{-6.6}$, and the mode of $\varepsilon_{\mathbf{t}}$ is decreased from $10^{-4.3}$ to $10^{-5.2}$.

\section{Extensions of Relative Pose from Affine Correspondences}
\label{sec:extension}

The minimal solver generation framework in this paper has great potential to solve many other relative pose estimation tasks. 
In this section, we extend this idea to many other configurations.

\begin{table}[bp]
\caption{Minimal solvers for 5DOF relative pose estimation of single cameras. ``+ka'' represents rotation angle is known. ``+f'' represents that the focal length is unknown. }
\begin{center}
{
	\setlength{\tabcolsep}{2.5pt}{
	{
		\begin{tabular}{|l|c|c|c|c|}
			\hline
			configuration & \#sol & template & \#sym. & action matrix \\ \hline
			2ac+cayl+mono & $20$ & $36\times 56$ & 0 & $20\times 20$ \\ \hline
			2ac+cayl+mono+ka & $20$ & $36\times 56$ & 0 & $20\times 20$ \\ \hline
			2ac+cayl+mono+f & $60$ & $213\times 243$ & 1 & $30\times 30$ \\ \hline
			2ac+cayl+mono+ka+f & $80$ & $300\times 340$ & 1 & $40\times 40$ \\ \hline
		\end{tabular}}
	}
}
\end{center}
\label{tab:complete_solution_mono}
\end{table}

\subsection{Known Rotation Angle}

When the rotation angle is known, the DOF of rotation can be reduced by one. There are minimal solvers for single cameras~\cite{li2013point,li2020relative,martyushev2020efficient} or multi-camera systems~\cite{martyushev2020efficient} using PCs. 
Our method can be easily extended to deal with known rotation angle cases. 

The known rotation angle prior imposes one constraint for rotation. 
Suppose the known rotation angle is $\theta$. By exploiting the relation between quaternion and axis-angle representation, a constraint can be constructed. For Cayley and quaternion parameterizations, the constraints are
\begin{align}
q_x^2+q_y^2+q_z^2 - \tan^2(\theta/2) = 0,
\end{align}
and
\begin{align}
q_x^2+q_y^2+q_z^2 - \tan^2(\theta/2) q_w^2 = 0,
\end{align}
respectively.
Note that 2~ACs provide $6$~independent constraints, and the known rotation angle case has $4$ and $5$~unknowns for a single camera and a multi-camera system, respectively. To construct a minimal configuration, we remove the excess constraints. 
For a single camera, the minimal configuration becomes 1AC+1PC. 

For single camera systems, the solver has $20$ solutions and the elimination template is $36\times 56$, see Table~\ref{tab:complete_solution_mono}.
For multi-camera systems, the statistics of the resulted solvers are shown in Table~\ref{tab:complete_solution_known_angle}. The solvers have $36$ solutions.

\begin{table} [tbp]
	\centering
	\caption{Minimal solvers for relative pose estimation with a known rotation angle for multi-camera systems. ``+ka'' represents rotation angle is known. ``+f'' represents that all focal lengths are unknown. 
	}
	\label{tab:complete_solution_known_angle}
	\setlength{\tabcolsep}{2pt}
	\begin{tabular}{|l|c|c|c|c|c|c|c|c|c|} 
		\hline
		\multirow{2}{*}{\centering configuration}  & \multicolumn{2}{c|}{equations $\mathcal{E}_1$} &  \multicolumn{2}{c|}{equations $\mathcal{E}_1+\mathcal{E}_2$}  \\ 
		\cline{2-5} 
		&   \#sol  & template   &   \#sol    &   template  \\ 
		\hline
		2ac+cayl+(case $1\sim 5$)+ka  & $44$ &  $120\times 164$ & $36$ & $84\times 120$ \\ \hline
		2ac+cayl+(case $6$)+ka & $42$ & $98\times 164$ & $36$ &  $104\times 164$ \\ \hline
		2ac+cayl+(case $7$)+ka & $44$ & $120\times 164$ & $36$ & $84\times 120$ \\ \hline 
		\hline
		2ac+cayl+(case $1\sim 2$)+ka+f & $352$ & $2307\times 2659$ & $288$ & $1657\times 1945$ \\ \hline
		2ac+cayl+(case $3$)+ka+f & $352$ & $1837\times 2197$ & $288$ & $1648\times 1945$ \\ \hline
		2ac+cayl+(case $4$)+ka+f & $324$ & $2632\times 2998$ & $276$ & $2113\times 2407$ \\ \hline
		2ac+cayl+(case $5$)+ka+f & $322$ & $3165\times 3544$ & $274$ & $2648\times 2953$ \\ \hline
		2ac+cayl+(case $6$)+ka+f & $328$ & $2796\times 3282$ & $280$& $2503\times 3064$ \\ \hline
		2ac+cayl+(case $7$)+ka+f & $1$-dim & $-$ & $288$ & $2430\times 2819$  \\ \hline 
	\end{tabular}
\end{table}

\subsection{Relative Pose and Focal Length Problem}

Our method can also be extended to semi-calibrated cameras. 
To simplify the problem, we assume only focal length is unknown. For multi-camera systems, we further assume all perspective cameras have the same focal length.
Let $f$ be the unknown focal length of the perspective cameras in a multiple camera system. The intrinsic matrix becomes
\begin{align}
\K^{-1} = 
\begin{bmatrix}
w & 0 & 0 \\
0 & w & 0 \\
0 & 0 & 1
\end{bmatrix},
\end{align}
where $w = 1/f$.

Denote the $k$-th AC as $(\hat{\x}_k, \hat{\x}'_k, \A_k)$, where $\hat{\x}_k$ and $\hat{\x}'_k$ are the homogeneous coordinates in image plane for the first and the second views, respectively. 
\begin{align}
\x_k = \K^{-1} \hat{\x}_k, \quad \x'_k = \K^{-1} \hat{\x}'_k.
\end{align}

For a single camera, the three equations induced by one AC becomes
\begin{subequations}
	\begin{empheq}[left=\empheqlbrace]{align}
	& \hat{\x}_k'^T \K^{-T} \E \K^{-1} \hat{\x}_k = 0 \label{equ:image_sub1} \\
	& (\E^T \K^{-1} \hat{\x}'_k)_{(1:2)} = -\A_k^{T} (\E \K^{-1} \hat{\x}_k)_{(1:2)} \label{equ:image_sub2}
	\end{empheq}
\end{subequations}
There are 5DOF unknown pose and one unknown inverse focal length in this equation system. Thus the setting of two ACs is a minimal configuration.
There is a symmetry in this equations system: if $\{q_x, q_y, q_z, w\}$ is a solution, $\{-q_x, -q_y, q_z, -w\}$ is also a solution. In addition, an inequality $w \neq 0$ should be explicitly considered. Otherwise, there is one-dimensional extraneous roots. We consider this inequality by saturation method~\cite{larsson2017polynomial}. By the hidden variable technique, we obtain a solver with $60$ solutions. The elimination template is $213\times 243$ and the action matrix is $30\times 30$, see Table~\ref{tab:complete_solution_mono}.

For multi-camera systems, the solver can be constructed similarly. The minimal configuration is 2~ACs plus 1~PC, which has too many variants. To simplify the problem, we assume that the rotation angle is known, which makes 2AC being the minimal configuration. 
The statistics of the resulted solvers are shown in Table~\ref{tab:complete_solution_known_angle}.
It can be seen that both the solution number and elimination template are large. As a result, the solvers are neither efficient nor numerical stable. These results suggest that we should further reduce problem complexity to obtain a practical solver. For example, we may further reduce DOF of the relative pose by exploiting the known rotation axis prior or planar motion prior, or we can reduce the order of the polynomial equation system by assuming that only part of focal lengths is unknown. It is straightforward to obtain solvers for these cases using our framework.

\section{Experiments}
\subsection{Experiments for Single Cameras}

\begin{table}[tbp]
	\caption{Runtime comparison of solvers for single cameras (unit:~$\mu s$).}
	\begin{center}
	{
		{
		\begin{tabular}{lccc}
			\toprule
			method &  5pt-Nister~\cite{nister2004efficient} & 2AC-Barath~\cite{barath2018efficient} &  2AC-mono  \\
			\midrule
			mean time & 5.5 & 11.0 & 140.2 \\
			\bottomrule
		\end{tabular}}
		}
	\end{center}
	\label{tab:runtime1}
\end{table}

\begin{figure}[tbp]
	\begin{center}
		\subfigure[rotation error $\varepsilon_{\mathbf{R}, \text{chordal}}$]
		{
			\includegraphics[width=0.47\linewidth]{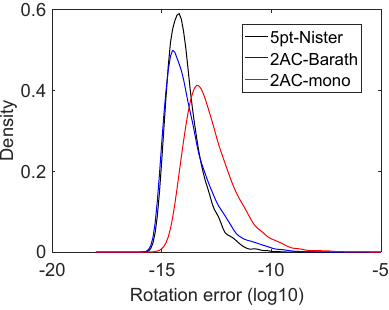}
		}
		\subfigure[translation error $\varepsilon_{\mathbf{t}}$]
		{
			\includegraphics[width=0.47\linewidth]{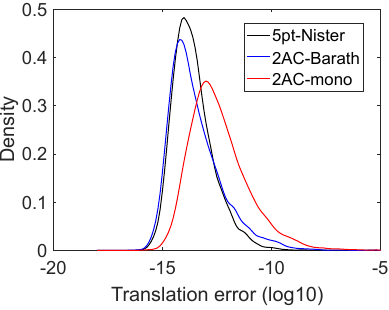}
		}
	\end{center}
	\vspace{-0.1in}
	\caption{Probability density functions over pose estimation errors on noise-free observations for single cameras. The horizontal axis represents the $\log_{10}$ errors and the vertical axis represents the density.}
	\label{fig:Numerical2}
\end{figure} 

In this paper, we obtain a minimal solver for single-cameras as a side product. 
The purpose of the following experiments is not to outperform the state-of-the-art method~\cite{barath2018efficient}. Instead, we illustrate that the obtained solver is practical.

For single cameras, the proposed solver is referred to as \texttt{2AC-mono}.
The proposed solvers are compared with state-of-the-art solvers including  \texttt{5pt-Nister}~\cite{nister2004efficient} and \texttt{2AC-Barath}~\cite{barath2018efficient}.
All the solvers are implemented in C++. The code of \texttt{5pt-Nister} is provided by the \texttt{PoseLib}\footnote{\url{https://github.com/vlarsson/PoseLib}}, and \texttt{2AC-Barath} is publicly available from the code of~\cite{barath2020making}.

Table~\ref{tab:runtime1} shows the average runtime of the solvers over $10,000$ runs for single cameras. 
Two methods based on 2AC are more efficient than \texttt{5pt-Nister} method. 
\texttt{2AC-Barath} takes $11$$\mu$s on average. 
The proposed \texttt{2AC-mono} takes about $0.14$ milliseconds, which is one order magnitude slower than \texttt{2AC-Barath}. 
Still, it is applicable for common scenarios and provides another option for this task.

Figure~\ref{fig:Numerical2} reports the numerical stability of the solvers on noise-free observations. The procedure is repeated $10,000$ times. The empirical probability density functions are plotted as the function of the $\log_{10}$ estimated errors $\varepsilon_{\mathbf{R}}$ and $\varepsilon_{\mathbf{t}}$. 
It can be seen that \texttt{2AC-Barath} and \texttt{5pt-Nister} have nearly the same numerical stability. The proposed \texttt{2AC-mono} has worse numerical stability than them. Still, it is applicable for practical applications since both the rotation error and translation error are below $1\times 10^{-8}$, and the modes of the rotation error and translation error are about $1\times 10^{-13}$.


%


\end{document}